\documentclass{article}

\PassOptionsToPackage{numbers, compress}{natbib}

    \usepackage[preprint]{neurips_2021}

\usepackage[utf8]{inputenc} 
\usepackage[T1]{fontenc}    
\usepackage{hyperref}       
\usepackage{url}            
\usepackage{booktabs}       
\usepackage{amsfonts}       
\usepackage{nicefrac}       
\usepackage{microtype}      
\usepackage{xcolor}         

\usepackage{graphicx} 
\usepackage{subfigure} 
\usepackage{wrapfig}
\usepackage{balance} 

\usepackage{multirow}               

\usepackage{helvet}

\usepackage{makecell}
\usepackage{color}
\usepackage{amsmath}
\usepackage{amssymb}

\usepackage{bm}

\usepackage{multirow}
\usepackage{booktabs}

\DeclareMathOperator*{\argmin}{arg\,min}

\def\ie{\mbox{\textit{i.e.}, }}
\def\eg{\mbox{\textit{e.g.}, }}

\def\PE{\mbox{PE}}
\def\pos{\mbox{pos}}
\def\sign{\mbox{sign}}

\def\mD{{\mathcal D}}

\def\mI{{\mathcal I}}

\def\mK{{\mathcal K}}
\def\mL{{\mathcal L}}

\def\mO{{\mathcal O}}
\def\mP{{\mathcal P}}
\def\mQ{{\mathcal Q}}

\def\mV{{\mathcal V}}
\def\mW{{\mathcal W}}
\def\mX{{\mathcal X}}
\def\mY{{\mathcal Y}}

\DeclareMathAlphabet\mathbfcal{OMS}{cmsy}{b}{n}

\def\0{{\bf 0}}
\def\1{{\bf 1}}

\def\bK{{\bm{K}}}

\def\bO{{\bm{O}}}

\def\bQ{{\bm{Q}}}

\def\bV{{\bm{V}}}
\def\bW{{\bm{W}}}
\def\bX{{\bm{X}}}
\def\bY{{\bm Y}}

\def\bb{{\bm b}}

\def\bu{{\bm u}}
\def\bv{{\bm v}}
\def\bw{{\bm w}}
\def\bx{{\bm x}}

\def\bz{{\bm z}}

\def\mmE{{\mathbb E}}
\def\mmP{{\mathbb P}}

\def\mmR{{\mathbb R}}

\usepackage{ntheorem}

\newtheorem{deftn}{Definition}
\newtheorem{thm}{Theorem}
\newtheorem*{*thm}{Theorem}

\newtheorem{lemma}{Lemma}
\newtheorem*{*lemma}{Lemma}

\newenvironment*{proof}{\textbf{Proof}\quad}{\hfill $\square$\par}

\usepackage{array}
\usepackage{tabu}
\makeatletter
\newcommand{\nipstophline}{%
	\noalign {\ifnum 0=`}\fi \hrule height 4pt
	\futurelet \reserved@a \@xhline
}
\newcommand{\nipsbottomhline}{%
	\noalign {\ifnum 0=`}\fi \hrule height 1pt
	\futurelet \reserved@a \@xhline
}
\usepackage{enumitem}

\def\red{\textcolor{red}}
\def\blue{\textcolor{blue}}

\usepackage[utf8]{inputenc} 

\usepackage[T1]{fontenc}
\usepackage{nicefrac}       

\DeclareRobustCommand{\cev}[1]{%
  {\mathpalette\do@cev{#1}}%
}
\newcommand{\do@cev}[2]{%
  \vbox{\offinterlineskip
    \sbox\z@{$\m@th#1 x$}
    \ialign{##\cr
      \hidewidth\reflectbox{$\m@th#1\vec{}\mkern4mu$}\hidewidth\cr
      \noalign{\kern-\ht\z@}
      $\m@th#1#2$\cr
    }%
  }%
}

\usepackage[d]{esvect}

\newcommand{\cao}[1]{{\color{black}#1}}

\newcommand{\tabincell}[2]{\begin{tabular}{@{}#1@{}}#2\end{tabular}}

\title{Video Super-Resolution Transformer}

\author{
	Jiezhang Cao$^{1}$ \qquad Yawei Li$^{1,}$\thanks{Co-first author.} \qquad Kai Zhang$^{1}$ \qquad Luc Van Gool$^{1,2}$ \\
	$^{1}$Computer Vision Lab, ETH Zürich, Switzerland \qquad $^{2}$KU Leuven, Belgium\\
	\{jiezhang.cao, yawei.li, kai.zhang, vangool\}@vision.ee.ethz.ch \\
}

\begin{document}
\maketitle

\begin{abstract}
Video super-resolution (VSR), with the aim to restore a high-resolution video from its corresponding low-resolution version, is a spatial-temporal sequence prediction problem. Recently, Transformer has been gaining popularity due to its parallel computing ability for sequence-to-sequence modeling. Thus, it seems to be straightforward to apply the vision Transformer to solve VSR. However, the typical block design of Transformer with a fully connected self-attention layer and a token-wise feed-forward layer does not fit well for VSR due to the following two reasons. First, the fully connected self-attention layer neglects to exploit the data locality because this layer relies on linear layers to compute attention maps. Second, the token-wise feed-forward layer lacks the feature alignment which is important for VSR since this layer independently processes each of the input token embeddings without any interaction among them. In this paper, we make the first attempt to adapt Transformer for VSR. Specifically, to tackle the first issue, we present a spatial-temporal convolutional self-attention layer with a theoretical understanding to exploit the locality information. For the second issue, we design a bidirectional optical flow-based feed-forward layer to discover the correlations across different video frames and also align features. Extensive experiments on several benchmark datasets demonstrate the effectiveness of our proposed method. The code will be available at \url{https://github.com/caojiezhang/VSR-Transformer}.

\end{abstract}

\section{Introduction}

Video super-resolution (VSR) refers to the task of enhancing low-resolution (LR) video to high-resolution (HR) one and has been successfully applied in some computer vision applications, such as video surveillance \cite{isobe2020video} and high-definition television \cite{tian2020tdan}. Generally, VSR can be formulated as a sequence modeling problem that can be solved by some sequence-to-sequence models, such as RNN \cite{chung2014empirical}, LSTM \cite{hochreiter1997long} 
and Transformer \cite{vaswaniattention}. 
Compared with RNN and LSTM, Transformer gains particular interest largely due to its recursion-free nature for parallel computing and modeling capacity for long-term dependencies of the input sequence.
Specifically, a Transformer block consists of two kinds of layers: a fully connected self-attention layer and a token-wise feed-forward layer, with skip connections in both layers.
Although Transformer has shown to work well for various computer vision tasks, directly applying it for VSR may suffer from two main limitations.

First, while the locality is well-known to be crucial for VSR, the fully connected self-attention (FCSA) layer neglects to leverage such information in a video sequence. 
Typically, most existing \cao{vision} Transformer methods \cao{(\eg ViT \cite{dosovitskiy2020image} and IPT \cite{chen2020pre}) split an image into several patches or tokens, which may damage the local spatial information \cite{li2021localvit} to some extent since the contents (\eg lines, edges, shapes, and even objects) are divided into different tokens.}
\cao{In addition}, this layer focuses on global interaction between the token embeddings by using several fully connected layers to compute attention maps which are irrelevant to local information.
As a result, the FCSA layer neglects to exploit the local information.
For VSR, it has been shown that temporal information is of vital importance.
When a region in some frames is occluded, the missing information can be recovered by other neighboring frames. 
Yet, it is still unclear for vision Transformer how to exploit the correlations among neighboring frames to improve the performance of VSR.

Second, the token-wise feed-forward layer cannot align features between video frames since this layer independently processes each of the input token embeddings without any interaction across them.
Although such interaction is contained in the FCSA layer, it ignores the feature propagation in video frames which contains rich bidirectional information. 
The feature propagation and feature alignment are crucial components in VSR since \cao{they help to exploit and align such information in a video sequence.}
However, most existing vision Transformer methods \cite{wu2020adversarial} lack both feature propagation and alignment. 
Without feature propagation, these methods generally fail to jointly capture past and future information.
As a result, the features may be unaligned to each frame, which leads to inferior performance. 
In short, explicit feature alignment and propagation mechanism are worth studying to improve the VSR performance.
In addition, it is impractical to directly use the fully connected linear layer to VSR since it has an expensive computational cost for many high-dimensional frames.
Therefore, it is very necessary and important to explore a new feed-forward layer to perform the feature propagation and alignment for VSR.

To address these two limitations, we propose a new Transformer for video super-resolution, called VSR-Transformer, which consists of a spatial-temporal convolutional self-attention (STCSA) layer and a bidirectional optical flow-based feed-forward (BOFF) layer. 
First, the STCSA layer exploits the locality from all token embeddings by introducing convolutional layers.
Then the BOFF layer learns the spatial-temporal information with the feature propagation and alignment.

The main contributions of this paper are summarized as follows:
\begin{itemize}[leftmargin=*]\setlength{\itemsep}{0pt}
    \item We propose a spatial-temporal convolutional attention layer to exploit the locality and spatial-temporal data information through different layers.
    We provide a theoretical analysis to support that our layer has an advantage over the fully connected self-attention layer.
    \item We design a new bidirectional optical flow-based feed-forward layer to use the interaction across all frame embeddings. This layer is able to improve the performance of VSR by performing feature propagation and alignment.
    Moreover, it alleviates the limitations of the traditional Transformer.
    \item We provide extensive experiments on several benchmark datasets to demonstrate the effectiveness of VSR-Transformer against state-of-the-art methods, especially for the limited number of frames.
\end{itemize}

\section{Related Work}
With the help of deep neural networks \cite{cao2018lccgan, cao2019multi, guo2020closed}, super-resolution (SR) which aims to reconstruct HR images/videos from LR images/videos has drawn significant attention.
Recently, several attempts use Transformer to solve SR.
For example, TTSR \cite{yang2020learning} proposes a texture Transformer by transferring HR textures from the reference image to the LR image.
IPT \cite{chen2020pre} develops a new pre-trained model to study the low-level computer vision task, including SR.
However, it is non-trivial and difficult to directly extend these Transformer-based image SR methods to VSR. 
Generally, existing VSR approaches \cite{chan2020basicvsr,isobe2020rsdn, isobe2020revisiting,jo2018deep, tian2020tdan,wang2019edvr} can be mainly divided into two frameworks: sliding-window and recurrent.

\paragraph{Video super-resolution.}
Earlier sliding window methods \cite{caballero2017real, tao2017detail, xue2019video} predict the optical flow between LR frames and perform the alignment by spatial warping. 
To improve the performance of VSR, 
TDAN \cite{tian2020tdan} uses deformable convolutions (DCNs) \cite{dai2017deformable, wang2019deformable} to adaptively align the reference frame and each supporting frame at the feature level. 
Motivated by TDAN, EDVR \cite{wang2019edvr} propose a video restoration framework by further enhancing DCNs to improve the feature alignment in a multi-scale fashion. 
This method first devises a pyramid cascading and deformable (PCD) alignment module to handle large motions and then uses a temporal and spatial attention (TSA) module to fuse important features.
To implicitly handle motions, DUF \cite{jo2018deep} leverages dynamic upsampling filters. 

In addition, some approaches take a recurrent framework. 
For example, RSDN \cite{isobe2020rsdn} proposes a recurrent detail-structural block to exploit previous frames to super-resolved the current frame.
RRN \cite{isobe2020revisiting} adopts a residual mapping between layers with identity skip connections to stabilize the training of
RNN and meanwhile to boost the super-resolution performance.
BasicVSR \cite{chan2020basicvsr} adopts a typical
bidirectional recurrent network coupled with a simple optical flow-based feature alignment for VSR.

\section{Preliminary and Problem Definition}
\paragraph{Notation.}
We use a calligraphic letter $\mX$ or $\mD$ to denote a sequence data or distribution, a bold upper case letter $\bX$ to denote a matrix, a bold lower case letter $\bx$ to denote a vector, a lower case letter $x$ to denote an element of a matrix. 
Let $\sigma_1(\cdot)$ be the softmax operator applied to each column of the matrix, \ie the matrix has non-negative elements with each column summing to 1.
Let $\sigma_2(\cdot)$ be a ReLU activation function, 
and let $\phi(\cdot)$ be a layer normalization function.
Let $[T]$ be a set $\{1, \ldots, T\}$.

To develop our method, we first give some definitions of the function distance and $k$-pattern function.
\begin{deftn}\textbf{\emph{(Function distance)}}
    Given a functions $f: \mmR^{d{\times}n} \to \mmR^{d{\times}n}$ and a target function $f^*: \mmR^{d{\times}n} \to \mmR^{d{\times}n}$, we define a distance between these two function as:
    \begin{align}
        \mL_{f^*, \mD}(f):= \mmE_{\bX \sim \mD} \left[ \ell(f(\bX), f^*(\bX)) \right].
    \end{align}
\end{deftn}
For a ground-truth $\bY = f^*(\bX)$, we denote the loss by $\mL_{\mD}(f)$.
To capture the locality of data, we define the $k$-pattern function as follows.

\begin{deftn}\textbf{\emph{($k$-pattern \cite{malach2021computational})}}
    \label{def:k_pattern}
    A function $f{:} \mX {\to} \mY$ is a $k$-pattern if for some $g{:} \{{\pm}1\}^{k} {\to} \mY$ and index $j^*${:} $f(\bx) = g(x_{j^* \ldots j^*{+}k})$.
    We call a function $h_{\bu, \bW}(\bx) = \sum_{j} \langle \bu^{(j)}, \bv^{(j)}_{\bW} \rangle$ can learn a $k$-pattern function from a feature $\bv^{(j)}_{\bW}$ of data $\bx$ with a layer $\bu^{(j)}\in\mmR^{q}$ if for $\epsilon>0$, we have $\mL_{f, \mD}(h_{\bu, \bW}) \leq \epsilon$.
\end{deftn}
Note that the feature $\bv^{(j)}_{\bW}$ can be learned by a convolutional attention network or a fully connected attention network parameterized by $\bW$.
This definition takes a vector as an example, and it can be extended to a matrix or a tensor.
In Definition \ref{def:k_pattern}, a $k$-pattern depends only on a small pattern of consecutive bits of the input.
Any function can learn the locality of data means that we should learn a hypothesis $h$ such that $\mL_{f, \mD}(h_{\bu, \bW}) \leq \epsilon$.

\paragraph{Video super-resolution.}
Let $\mD$ be a distribution of videos, and let $\{ \bV_{1}, \ldots, \bV_{T}\}\sim \mD$ be a low-resolution (LR) video sequence, where $\bV_t \in \mmR^{3{\times} W {\times} H}$ is the $t$-th LR frame.
We use a feature extractor to learn features $\mX = \{ \bX_{1}, \ldots, \bX_{T}\}$ from LR video frames, where $\bX_t \in \mmR^{C{\times} W {\times} H}$ is the $t$-th feature.
The goal of VSR is to learn a non-linear mapping $F$ to reconstruct high-resolution (HR) frames $\hat{\mY}$ by fully utilizing the spatial-temporal information across the sequence, \ie
\begin{align}
    \widehat{\mY} \triangleq \left(\widehat{\bY}_{1}, \ldots, \widehat{\bY}_{T}\right) = F (\bV_{1}, \ldots, \bV_{T}),
\end{align}
Given the ground-truth HR frames $\mY {=} \{ \bY_{1}, \ldots, \bY_{T}\}$, where $\bY_t$ is the $t$-th HR frame.
Then we minimize a loss function between the generated HR frame $\widehat{\bY}_t$ and the ground-truth HR frame $\bY_t$, \ie
\begin{align}
    \widehat{F} = \argmin\limits_{F} \mL_{\mD} \left(F\right) \triangleq 
    \widehat{\mmE}_{\mD, t \in [T]} \left[  d \left( \widehat{\bY}_t, \bY_t \right) \right],
\end{align}
where $\widehat{\mmE}[\cdot]$ is an empirical expectation, $d(\cdot, \cdot)$ is a distance metric, \eg $\ell_1$-loss, $\ell_2$-loss and Charbonnier loss \cite{wang2019edvr}.
For the VSR problem, one can use a sequence modeling method, such as RNN \cite{chung2014empirical}, LSTM \cite{hochreiter1997long} and Transformer \cite{vaswaniattention}.
In practice, Transformer gains particular interest
since it avoids recursion and thus allows parallel computing in practice.

\paragraph{Transformer block.}
A Transformer block is a sequence-to-sequence function, which consists of a self-attention layer and a token-wise feed-forward layer with both having a skip connection.
Specifically, given an input $\bX \in \mmR^{d \times n}$ consisting of $d$-dimensional embeddings of $n$ tokens, the Transformer block map a sequence $\mmR^{d \times n}$ to another sequence $\mmR^{d \times n}$, respectively, \ie
\begin{align}
    f_{1}(\bX) &= \phi \left( \bX + \sum\nolimits_{i=1}^h \bW_o^i (\bW_v^i \bX) \sigma_1((\bW_k^i \bX)^{\top}(\bW_q^i \bX)) \right),\\
    f_{2}(\bX) &= \phi \left( f_{1}(\bX) + \bW_2 \sigma_2(\bW_1 \cdot f_{1}(\bX) + \bb_1 \1_n^{\top}) + \bb_2 \1_n^{\top}\right), \label{eqn:FC_feedforward}
\end{align}
where 
$\bW^i_v, \bW^i_k, \bW^i_q \in \mmR^{m \times d}$ are linear layers mapping an input to value, key and query, respectively.
Also, $\bW_o^i {\in} \mmR^{d {\times} m}$, $\bW_1 {\in} \mmR^{r {\times} d}, \bW_2 {\in} \mmR^{d {\times} r}$ are linear layers, and $\bb_1 \in \mmR^r, \bb_2 {\in} \mmR^d$ are bias.
Here, $h$ is the number of heads, $m$ is the head size, and $r$ is the hidden layer size of the feed-forward layer.

However, it is non-trivial to apply Transformer to VSR due to some intrinsic limitations.
i) The fully connected self-attention layer neglects to leverage locality information in a video.  
ii) The token-wise feed-forward layer independently processes each of the input token embeddings, leading to misaligned features.
To address these, we propose a new Transformer for VSR.

\begin{figure}[t]
\centering
{
	\includegraphics[width=0.95\linewidth]{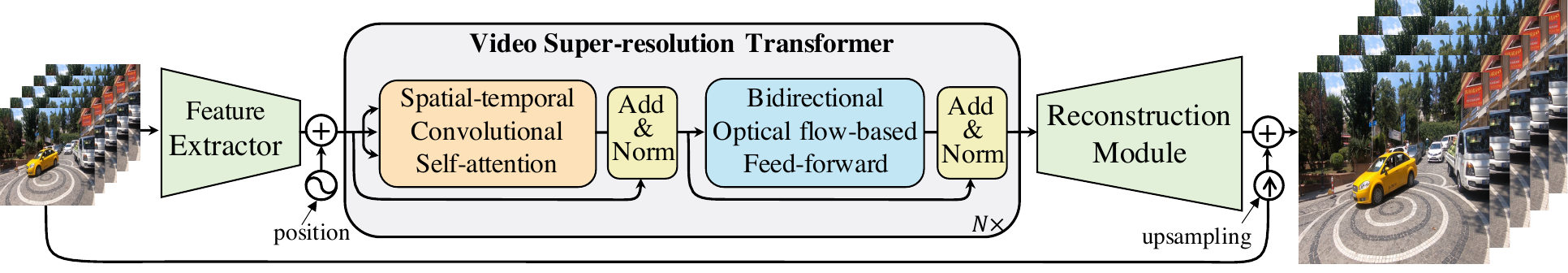}
	\caption{The framework of video super-resolution Transformer. Given a low-resolution (LR) video, we first use an extractor to capture features of the LR videos.
    Then, a spatial-temporal convolutional self-attention and an optical flow-based feed-forward network model a sequence of continuous representations.
    Note that these two layers both have skip connections.
    Last, the reconstruction network restores a high-resolution video from the representations and the upsampling frames.
	}
	\label{fig: framework}
}
\end{figure}
	
\section{Video Super-Resolution Transformer}

In this paper, we aim to propose a new Transformer for the video super-resolution problem, called VSR-Transformer. 
As illustrated in Figure \ref{fig: framework}, our proposed method consists of a feature extractor, a Transformer encoder, and a reconstruction network.
Specifically, given a sequence of videos, we first use a stack of residual blocks to extract features of the videos.
Then, the VSR-Transformer encoder maps the features to a sequence of continuous representations.
Last, the reconstruction module restores a high-resolution video from the representations.

\subsection{Spatial-Temporal Convolutional Self-attention}\label{sec:stcsa}

To verify the drawback of the fully connected self-attention (FCSA) layer, we first provide a theoretical analysis for whether it learns $k$-patterns with gradient descent. 
Let $\mD$ be the uniform distribution, and let $f(\bX)=\Pi_{i,j\in\mI} x_{i,j}$, where $\mI$ is some set of $k$ consecutive bits.
When a fully connected attention layer is initialized as a permutation invariant distribution, the initial gradient is very small.
Specifically, we have the following theorem, 
\begin{thm}\label{thm:fcsa}
    Let $n$ be the size of image and $q$ be the size of $u$. We assume $m=1$ and $|u_i|\leq 1$. and the weights are initialized as some permutation invariant distribution over $\mmR^n$, and for all $\bx$ we have $h_{\bu, \bW}^{\emph{FCSA}}(\bx) \in [-1, 1]$ which satisfies Definition \ref{def:k_pattern}. Then, the following holds:
    \begin{align}
        \mmE_{\bW \sim \mW} \left\| \frac{\partial}{\partial \bW} \mL_{f, \mD}\left(h_{\bu, \bW}^{\emph{FCSA}}\right) \right\|_2^2 
        &\leq q n \min\left\{\begin{pmatrix} n{-}1 \\ k \end{pmatrix}^{{-}1}, \begin{pmatrix} n{-}1 \\ k{-}1 \end{pmatrix}^{{-}1}  \right\}.
    \end{align}
\end{thm}
\begin{proof}
    Please see the proofs in the supplementary materials.
\end{proof}

From this theorem, the initial gradient is small if $k=\Omega(\log n)$.
When $q$ is not sufficiently large, the fully connected attention layer may result in the gradient vanishing issue. 
It implies that the gradient descent will be “stuck” upon the initialization, and thus will fail to learn the $k$-pattern function.
Therefore, the fully connected self-attention layer cannot use the spatial information of each frame since the local information is not encoded in the embeddings of all tokens.
Moreover, this issue may become more serious when directly using such layers in video super-resolution.

To address this, we propose a new spatial-temporal convolutional self-attention (STCSA) layer.
As illustrated in Figure \ref{fig: attention}, given the feature maps of input video frames $\mX \in \mmR^{T{\times}C{\times}W{\times}H}$, we use three independent convolutional neural networks $\mW_q, \mW_k $ and $\mW_v$ to capture the spatial information of each frame.
Here, the kernel size is $3 \times 3$, the stride is 1 and padding is 1.
Different from Vision Transformer (ViT) \cite{dosovitskiy2020image}, it uses a linear projection to extract several patches when taking an image as an input.
In contrast, motivated by COLA-Net \cite{mou2021cola}, we use the unfold operation to extract sliding local patches with stride $s$ and patch size of $W_p {\times} H_p$ after inputting each frame to our Transformer.
Then, we obtain three groups of 3D patches, and each group has $N{=}TWH/(W_p H_p)$ patches with the dimension of each patch being $d{=}C{\times}W_p{\times}H_p$.
Then, we generate query, key and value $\mQ, \mK, \mV \in \mmR^{TN{\times}C{\times}W_p{\times}H_p}$, \ie
\begin{align}
    \mQ = \kappa_1(\mW_q * \mX),\quad \mK = \kappa_1(\mW_k * \mX),\quad \mV = \kappa_1(\mW_v * \mX),
\end{align}
where $\kappa_1(\cdot)$ is a unfold operation.
Next, we reshape each patch into a new query matrix and key matrix 
$\bQ=\tau (\mQ)$ and $\bK =\tau (\mK)$ with the size of $d{\times}N$, where $\tau(\cdot)$ is a composition of the unfold and reshape operation.
Then, we calculate the similarity matrix $\sigma_1 (\bQ^{\top} \bK)$ and aggregate with the value ${\mV}$ to obtain a feature map.
Note that the similarity matrix is related to all embedding tokens of the whole video frames.
Therefore, it implies that the spatial-temporal information is captured in our proposed layer.
Last, we use the fold operation $\kappa_2(\cdot)$ to combine these tensors of updating sliding local patches into a feature map with the size of $C{\times}T{\times}W{\times}H$ and obtain the final feature map by using an output layer $\mW_o$.
This process can be viewed as the inverse process of the unfold operation.
Summarizing the above, we define the spatial-temporal convolutional self-attention layer as:
\begin{align}
    f_1(\mX) &= \phi \left( \mX + \sum\nolimits_{i=1}^h \mW_o^i * \kappa_2\left( \underbrace{\kappa_1(\mW_v^i * \mX)}_{\mV} \sigma_1(\underbrace{\kappa_1^{\tau}(\mW_k^i * \mX)}_{\mK} ~\!\!\!^{\top} \underbrace{\kappa_1^{\tau}(\mW_q^i * \mX)}_{\mQ}) \right) \right),
\end{align}
where $\kappa_1^{\tau}(\cdot){=} \tau{\circ} \kappa_1(\cdot)$ is a composition of the reshape operation $\tau$ and the fold operation $\kappa_1$.
In the experiment, we use a single head (\ie $h{=}1$) to achieve good performance.
By using our spatial-temporal convolutional attention layer, we next provide the following theorem for the STCSA layer about how to learn $k$-patterns with gradient descent.

\begin{thm}\label{thm:stcsa}
    Assume we initialize each element of weights uniformly drawn from $\{\pm 1/k\}$.
    Fix some $\delta >0$, some $k$-pattern $f$ and some distribution $\mD$. Then is $q>2^{k+3} \log (2^k/\delta)$, and let $h_{\bu^{(s)}, \mW^{(s)}}^{\emph{STCSA}}$ be a function satisfying Definition \ref{def:k_pattern},
    with probability at least $1-\delta$ over the initialization, when training a spatial-temporal convolutional self-attention (STCSA) layer using gradient descent with $\eta$, we have
    \begin{align}
        \frac{1}{S} \sum_{s=1}^{S} \mL_{f, \mD} \left(h_{\bu^{(s)}, \mW^{(s)}}^{\emph{STCSA}} \right) \leq \eta^2 S^2 nk^{5/2} 2^{k+1} + \frac{k^2 2^{2k+1}}{q \eta S} +  \eta nqk.
    \end{align}
\end{thm}
\begin{proof}
    Please see the proofs in the supplementary materials.
\end{proof}
From the theorem, the loss $\mL_{f, \mD} (h_{\bu^{(s)}, \mW^{(s)}}^{\emph{STCSA}})$ can be small with finite $S$ steps in the optimization, and thus the spatial-temporal convolutional self-attention layer using gradient descent is able to learn the $k$-pattern function.
It implies that our proposed layer captures the locality of each frame.
These results verify that our spatial-temporal convolutional self-attention layer achieves an advantage over the fully connected self-attention layer.

\begin{figure}[t]
\centering
{
	\includegraphics[width=1\linewidth]{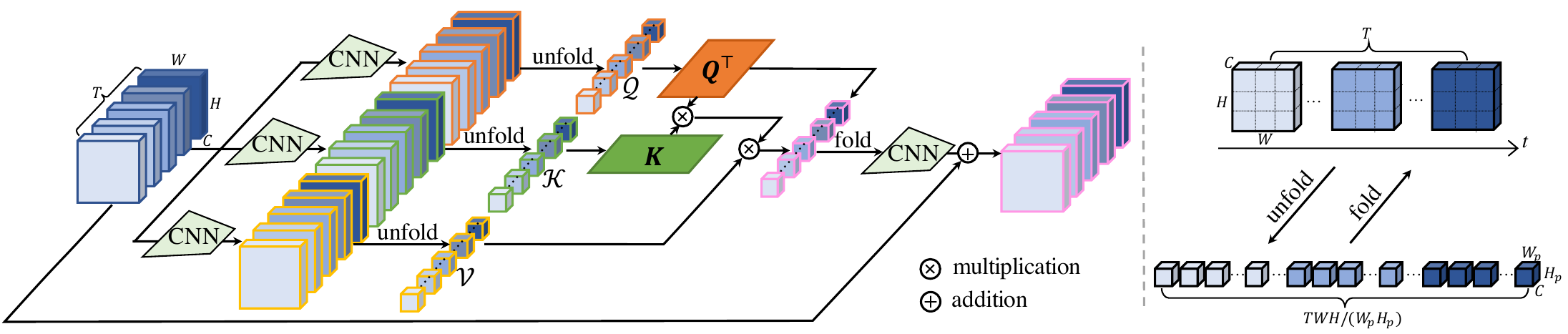}
	\caption{Illustration of the spatial-temporal convolutional self-attention. 
	The unfold operation is to extract sliding local patches from a batched input feature map, while the fold operation is to combine an array of sliding local patches into a large feature map.
	}
	\label{fig: attention}
}
\end{figure}

\paragraph{Spatial-temporal positional encoding.}
The architecture of the proposed VSR-Transformer is permutation-invariant, while the VSR task requires precise spatial-temporal position information.
To address this, we propose to use 3D fixed positional encodings  \cite{wang2020end} and add them to the input of the attention layer.
Specifically, the positional encodings contain two spatial positional information (\ie horizontal and vertical) and one temporal positional information.
Then, we formulate the spatial-temporal positional encoding (PE) as follows:
\begin{align}
    \PE(\pos, i) = 
    \left\{
    \begin{array}{ll}
    \!\!\!\sin(\pos \cdot \alpha_k), & \text{for}~~ i = 2k, \\
    \!\!\!\cos(\pos \cdot \alpha_k), & \text{for}~~ i = 2k+1,
    \end{array} 
    \right.
\end{align}
where $\alpha_k{=}1/10000^{2k/\frac{d}{3}}$, $k$ is an integer in $[0, d/6)$, `$\pos$' is the position in the corresponding dimension, and $d$ is the size of the channel dimension.
Note that the dimension $d$ should be divisible by 3 since the positional encodings of the three dimensions should be concatenated to form the final $d$ channel positional encodings.

\begin{figure}[t]
\centering
{
	\includegraphics[width=1\linewidth]{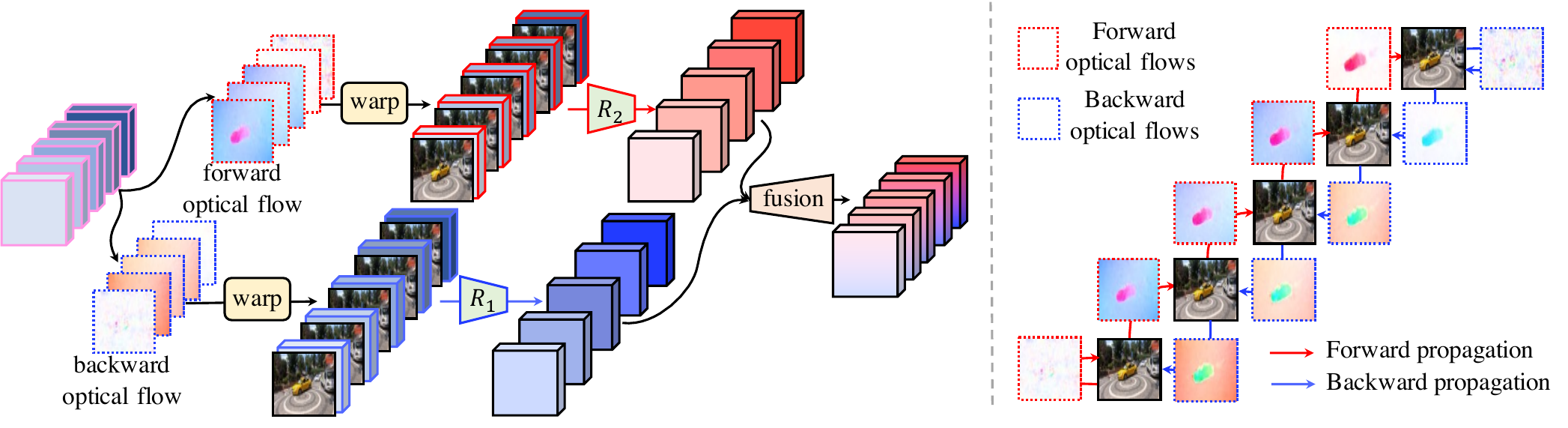}
	\caption{Illustration of the bidirectional optical flow-based feed-forward layer. 
	Given a video sequence, we first bidirectionally estimate the forward and backward optical flows and wrap the feature maps with the responding optical flows.
	Then we learn a forward and backward propagation network to produce two sequences of features from concatenated wrapped features and LR frames.
	Last, we fusion these two feature sequences into one feature sequence.
	}
	\label{fig:feedforward}
}
\end{figure}

\subsection{Bidirectional Optical Flow-based Feed-Forward}
The fully connected feed-forward layer in the traditional Transformer consists of two linear layers with a ReLU activation in between, which is applied to each token separately and identically.
In this way, this layer neglects to exploit the correlations among tokens of different frames, which may lead to poor performance. 
To address this, we propose to model the correlations among all frames. 
Motivated by flow-based methods \cite{chan2020basicvsr}, we propose a bidirectional optical flow-based feed-forward layer by using optical flow for spatial alignment, as shown in Figure \ref{fig:feedforward} (left).
Specifically, given feature maps $\mX$ outputted by the spatial-temporal convolutional self-attention layer, we first learn bidirectional optical flow $\cev{\mO}$ and $\vec{\mO}$ between neighboring frames. 
Then, we obtain backward features $\cev{\mX}$ and forward features $\vec{\mX}$ by using a warping function $\omega(\cdot, \cdot)$ along with the backward and forward propagation, \ie 
\begin{align}
    \cev{\mX} = \omega \left(\mX, \cev{\mO} \right), \quad \vec{\mX} = \omega \left(\mX, \vec{\mO} \right),
\end{align}
where $\cev{\mO}, \vec{\mO} \in \mmR^{T{\times}2{\times}W{\times}H}$ are backward and forward optical flows, respectively.
In practice, we use SPyNet \cite{ranjan2017optical} as a function $s(\cdot, \cdot)$ to bidirectionally estimate the optical flows, \ie 
\begin{align}
    \cev{\bO}_t = 
    \left\{
    \begin{array}{ll}
    \!\!\!s(\bV_{1}, ~~\bV_1),       & \text{if}~~ t = 1, \\
    \!\!\!s(\bV_{t\!{-}\!1}, \bV_t), & \text{if}~~ t \in (1, T],
    \end{array} 
    \right.
    \qquad
    \vec{\bO}_t = 
    \left\{
    \begin{array}{ll}
    \!\!\!s(\bV_{t\!{+}\!1}, \bV_{t}), & \text{if}~~ t \in [1, T), \\
    \!\!\!s(\bV_T, ~\bV_T),            & \text{if}~~ t = T,
    \end{array} 
    \right.
\end{align}
where $\cev{\bO}_t$ and $\vec{\bO}_t$ are the $t$-th element of $\cev{\mO}$ and $\vec{\mO}$, respectively. 
Note that the function $s(\cdot, \cdot)$ is pre-trained and updated in the training.  
Here, we estimate the identical optical flow at the start and end of a video for the backward and forward propagation, respectively, as shown in Figure \ref{fig:feedforward} (right).
Then, we aggregate the video frames and warped feature maps to maintain the video information.
To learn the correlation among neighboring frames, we propose to use convolutional backward and forward propagation networks.
We modify the fully connected feed-forward layer (\ie Eqn.  (\ref{eqn:FC_feedforward})) as:
\begin{align}
    f_2(\mX) &= \phi \left( f_1(\mX) + \rho \left(\underbrace{\cev{\bW}_1 * \sigma_2 \left(\cev{\bW}_2 * \left[\mV, \cev{\mX}\right]\right)}_{\text{backward propagation}} + \underbrace{\vec{\bW}_1 * \sigma_2 \left(\vec{\bW}_2 * \left[\mV, \vec{\mX}\right]\right)}_{\text{forward propagation}} \right) \right),
    \label{eqn:f2_twolayer}
\end{align}
where $\rho(\cdot)$ is a fusion module. 
Note that we take a two-layered network as an example, $\cev{\bW}_1$ and $\cev{\bW}_2$ are the weights of the backward propagation network, and $\vec{\bW}_1$ and $\vec{\bW}_2$ are the weights of the forward propagation network.
In practice, we extend the case of two-layered networks to multi-layered neural networks $R_1$ and $R_2$, then we rewrite Eqn. (\ref{eqn:f2_twolayer}) as follows: 
\begin{align}
    f_2(\mX) = \phi \left( f_1(\mX) + \rho \left( {R}_1\left(\mV, \cev{\mX}\right) + {R}_2\left(\mV, \vec{\mX}\right) \right) \right),
\end{align}
where $R_1$ and $R_2$ are flexible networks, and we set them to be a stack of Residual ReLU networks in the experiment.
Compared with ViT \cite{dosovitskiy2020image}, our model is able to capture the correlation among different frames.
Different from BasicVSR \cite{chan2020basicvsr}, it recurrently estimates the optical flows and features.
In contrast, our VSR-Transformer avoids recursion and thus allows parallel computing.


\newpage
\section{Experiments}\label{sec:experiment}

\begin{table*}[!t]
	\caption{Quantitative comparison (PSNR/SSIM) on \textbf{REDS4} for $4\times$ VSR. The results are tested on RGB channels. {Red} and {blue} indicate the best and the second best performance, respectively. `$\dagger$' means a method trained on \textbf{5 frames} for a fair comparison.}
	\label{tab:reds}
	\begin{center}
		\scalebox{0.8}{
			\begin{tabular}{l|c||c|c|c|c|c}
				\hline
				Method & Params (M) & Clip\_000 & Clip\_011 & Clip\_015 & Clip\_020 & Average (RGB)  \\ \hline
				Bicubic & - & 24.55/0.6489 & 26.06/0.7261 & 28.52/0.8034 & 25.41/0.7386 & 26.14/0.7292  \\
				RCAN~\cite{zhang2018image} & - & 26.17/0.7371 & {29.34}/{0.8255} & {31.85}/{0.8881} & {27.74}/{0.8293} & {28.78}/0.8200  \\ \hline
				TOFlow~\cite{xue2019video} & - & 26.52/0.7540 & 27.80/0.7858 & 30.67/0.8609 & 26.92/0.7953 & 27.98/0.7990 \\
				DUF~\cite{jo2018deep} & 5.8 & {27.30}/{0.7937} & 28.38/0.8056 & 31.55/0.8846 & 27.30/0.8164 & 28.63/{0.8251} \\
				{EDVR-M}~\cite{wang2019edvr} & 3.3 & {27.75}/{0.8153} & {31.29}/{0.8732} & {33.48}/{0.9133} & {29.59}/{0.8776} & {30.53}/{0.8699} \\ 
				{EDVR-L}~\cite{wang2019edvr} & 20.6 & \blue{\bf{28.01}}/\blue{\bf{0.8250}} & \blue{\bf{32.17}}/\blue{\bf{0.8864}} & \blue{\bf{34.06}}/\red{\bf{0.9206}} & \blue{\bf{30.09}}/\blue{\bf{0.8881}} & \blue{\bf{31.09}}/\blue{\bf{0.8800}} \\ 
				{BasicVSR$^\dagger$~\cite{chan2020basicvsr}} & 6.3 & 27.67/{0.8114} & 31.27/0.8740 & 33.58/0.9135 & {29.71}/{0.8803} & 30.56/0.8698  \\ 
				{IconVSR$^\dagger$~\cite{chan2020basicvsr}} & 8.7 & {27.83}/{0.8182} & 31.69/0.8798 & {33.81}/{0.9164} & 29.90/{0.8841} & 30.81/{0.8746}  \\
				\hline
				\textbf{VSR-Transformer} & 32.6 & \red{\bf{28.06}}/\red{\bf{0.8267}} & \red{\bf{32.28}}/\red{\bf{0.8883}} & \red{\bf{34.15}}/\blue{\bf{0.9199}} & \red{\bf{30.26}}/\red{\bf{0.8912}} & \red{\bf{31.19}}/\red{\bf{0.8815}}  \\ \hline
			\end{tabular}
		}
		
	\end{center}
\end{table*}

\begin{figure}[t]
\centering
{
	\includegraphics[width=1\linewidth]{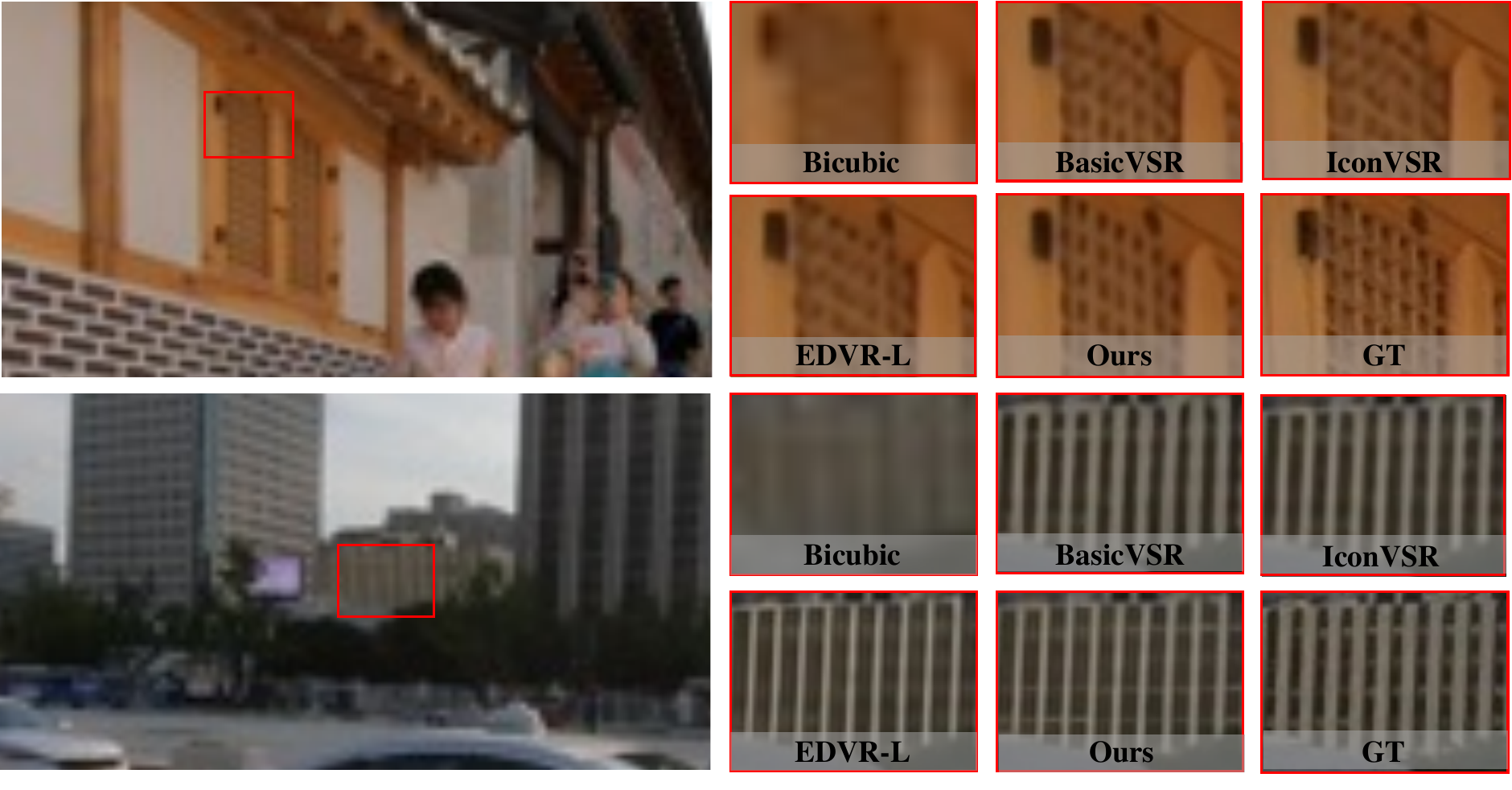}
	\caption{Qualitative comparison on the REDS4 dataset for $4\times$ VSR. Zoom in for the best view.
	}
	\label{fig: reds1}
}
\end{figure}

\textbf{Training datasets.}
\textbf{(i) REDS} \cite{nah2019ntire} contains 240 training clips, 30 validation clips and 30 testing clips, where each with 100 consecutive frames.
According to EDVR \cite{wang2019edvr}, we use REDS4 as the test set which contains the 000, 011, 015, and 020 clips.
The remaining training and validation clips are
regrouped as our training dataset which has 266 clips. 
For fair comparisons, all VSR models are trained on 5 frames.
\textbf{(ii) Vimeo-90K} \cite{xue2019video} 
consists of 4,278 videos with 89,800 high-quality video clips (\ie 720p or higher) collected from Vimeo.com, which covers a large variety of scenes and actions. 
We use Vid4 \cite{liu2013bayesian} and Vimeo-90K-T \cite{xue2019video} as test sets.


\textbf{Evaluation metrics.}
We use Peak Signal-to-Noise Ratio (PSNR) and Structural Similarity Index (SSIM) \cite{wang2004image} to evaluate the quality of images generated by the VSR methods. 
To measure the efficiency of different networks, we also compare the model sizes. 

\textbf{Experiment details.}
We compare our VSR-Transformer with the following state-of-the-art VSR methods: RCAN \cite{zhang2018image},
VESPCN \cite{caballero2017real}, SPMC \cite{tao2017detail}, TOFlow \cite{xue2019video}, FRVSR \cite{sajjadi2018frame}, DUF \cite{jo2018deep}, RBPN \cite{haris2019recurrent}, EDVR \cite{wang2019edvr}, BasicVSR \cite{chan2020basicvsr} and IconVSR \cite{chan2020basicvsr}.
Our experiments are implemented based on BasicSR \cite{wang2020basicsr}, with 8 NVIDIA TITAN RTX GPUs. 
We use Bicubic down-sampling to get LR images from HR images.
The channel size in each residual block is set to 64.
We set the number of Transformer blocks to be the number of frames.
We use Adam optimizer \cite{kingma2014adam} with $\beta_1{=}0.9, \beta_2{=}0.99$, and use Cosine Annealing \cite{loshchilov2016sgdr} to decay the learning rate from $2{\times}10^{-4}$ to $10^{-7}$.
More details of data augmentation, experiment settings, and the network architectures can be referred to in the supplementary materials.

\subsection{Resuts on REDS}\label{sec:reds}
We compare our proposed method with the state-of-the-art VSR methods on REDS.
For fair comparisons, we train BasicVSR and IconVSR \cite{chan2020basicvsr} with 5 frames to produce high resolution videos.
With the same amount of information, it is helpful to compare the performance of all VSR methods.

\textbf{Quantitative results.}
From Table \ref{tab:reds}, our method has the highest PSNR and comparable SSIM values, which verifies the superiority of our method.
When training with 5 frames, BasicVSR and IconVSR degrade severely and they are worse than EDVR.
It implies that their success largely derived from the aggregation of long-term sequence information.
More importantly, our model with 64 channels achieves better performance than EDVR-L with 128 channels.
On the other hand, although our model size is larger than other methods, it gains large improvement for VSR, especially for a small number of frames.
In practice, our model size can be smaller than directly using most existing Transformers in VSR due to many linear layers.
We leave this limitation on the model size in future work.

\begin{table*}[!t]
	\caption{Quantitative comparison (PSNR/SSIM) on \textbf{Vimeo-90K-T} for $4\times$ VSR. {Red} and {blue} indicate the best and the second best performance, respectively. }
	\label{tab:sr_vimeo90k}
	\begin{center}
		\tabcolsep=0.06cm
		\scalebox{0.7}{
		\begin{tabular}{c||c|c||c|c|c|c|c|c|c}
			\hline
			Average & Bicubic & RCAN~\cite{zhang2018image} &
			TOFlow~\cite{xue2019video} & DUF~\cite{jo2018deep} & RBPN~\cite{haris2019recurrent} & EDVR-L~\cite{wang2019edvr} & BasicVSR\cite{chan2020basicvsr} &
			IconVSR~\cite{chan2020basicvsr} &
			\textbf{Ours} 
			\\
			(Channel) & (1 Frame) & (1 Frame) & (7 Frames) & (7 Frames) & (7 Frames) & (7 Frames) & (7 Frames) & (7 Frames) & (7 Frames) \\ \hline
			RGB & 29.79/0.8483 & 33.61/0.9101 & 33.08/0.9054 & {34.33}/{0.9227} &-/-& \blue{\bf{35.79}}/\blue{\bf{0.9374}} & 35.31/0.9322 & 35.54/0.9347   & \red{\bf{35.88}}/\red{\bf{0.9380}} \\
			Y   & 31.32/0.8684 & 35.35/0.9251 & 34.83/0.9220 & 36.37/0.9387 & {37.07}/{0.9435} & \blue{\bf{37.61}}/\blue{\bf{0.9489}} & 37.18/0.9450 & 37.47/0.9476 & \red{\bf{37.71}}/\red{\bf{0.9494}} \\ \hline
		\end{tabular}}
	\end{center}
\end{table*}

\begin{table*}[!t]
	\caption{Quantitative comparison (PSNR/SSIM) on \textbf{Vid4} for $4\times$ VSR. {Red} and {blue} indicate the best and the second best performance, respectively. Y denotes the evaluation on Y channels. }
	\label{tab:sr_vid4}
	\begin{center}
		\tabcolsep=0.06cm
		\scalebox{0.8}{
		\begin{tabular}{l|c||c|c|c|c||c}
			\hline
			{Methods} & {Params (M)} & {Calendar (Y)} & {City (Y)} & {Foliage (Y)} & {Walk (Y)} & {Average (Y)} 
			\\
			\hline
			Bicubic                          & - & 20.39/0.5720 & 25.16/0.6028 & 23.47/0.5666 & 26.10/0.7974 & 23.78/0.6347  \\
			RCAN~\cite{zhang2018image}       & - & 22.33/0.7254 & 26.10/0.6960 & 24.74/0.6647 & 28.65/0.8719 & 25.46/0.7395 \\
			\hline
			VESPCN~\cite{caballero2017real}  & - & -/-   &   -/-    &   -/-    &   -/-   & 25.35/0.7557   \\
			SPMC~\cite{tao2017detail}        & - & 22.16/0.7465 & 27.00/0.7573 & 25.43/0.7208 & 28.91/0.8761 & 25.88/0.7752 \\
			TOFlow~\cite{xue2019video}       & - & 22.47/0.7318 & 26.78/0.7403 & 25.27/0.7092 & 29.05/0.8790 & 25.89/0.7651 \\
			FRVSR~\cite{sajjadi2018frame}    & 5.1 &   -/-   &   -/-    &   -/-    &   -/-   & 26.69/0.822   \\
			RBPN~\cite{haris2019recurrent}   & 12.2 & 23.99/0.807  & 27.73/0.803  & 26.22/0.757  & 30.70/0.909  & 27.12/0.818   \\
			EDVR-L~\cite{wang2019edvr}       & 20.6 & \blue{\bf{24.05}}/\red{\bf{0.8147}} & \red{\bf{28.00}}/\red{\bf{0.8122}} & \red{\bf{26.34}}/\blue{\bf{0.7635}} & \blue{\bf{31.02}}/\blue{\bf{0.9152}} & {27.35}/\blue{\bf{0.8264}} \\
			BasicVSR~\cite{chan2020basicvsr}& 6.3 &   -/-  &  -/-    &   -/-    &  -/-  & 27.24/0.8251  \\
			IconVSR~\cite{chan2020basicvsr}& 8.7 &   -/-  &  -/-    &   -/-   &  -/-  & \red{\bf{27.39}}/\red{\bf{0.8279}}  \\
			\hline
			\bf{VSR-Transformer (Ours)} & 43.8 &   \red{\bf{24.08}}/\blue{\bf{0.8125}}   &   \blue{\bf{27.94}}/\blue{\bf{0.8107}}    &   \blue{\bf{26.33}}/\red{\bf{0.7635}}    &   \red{\bf{31.10}}/\red{\bf{0.9163}}   &   \blue{\bf{27.36}}/0.8258    \\
			\hline
		\end{tabular}
		}
	\end{center}
\end{table*}

\begin{figure}[t]
\centering
{
	\includegraphics[width=1\linewidth]{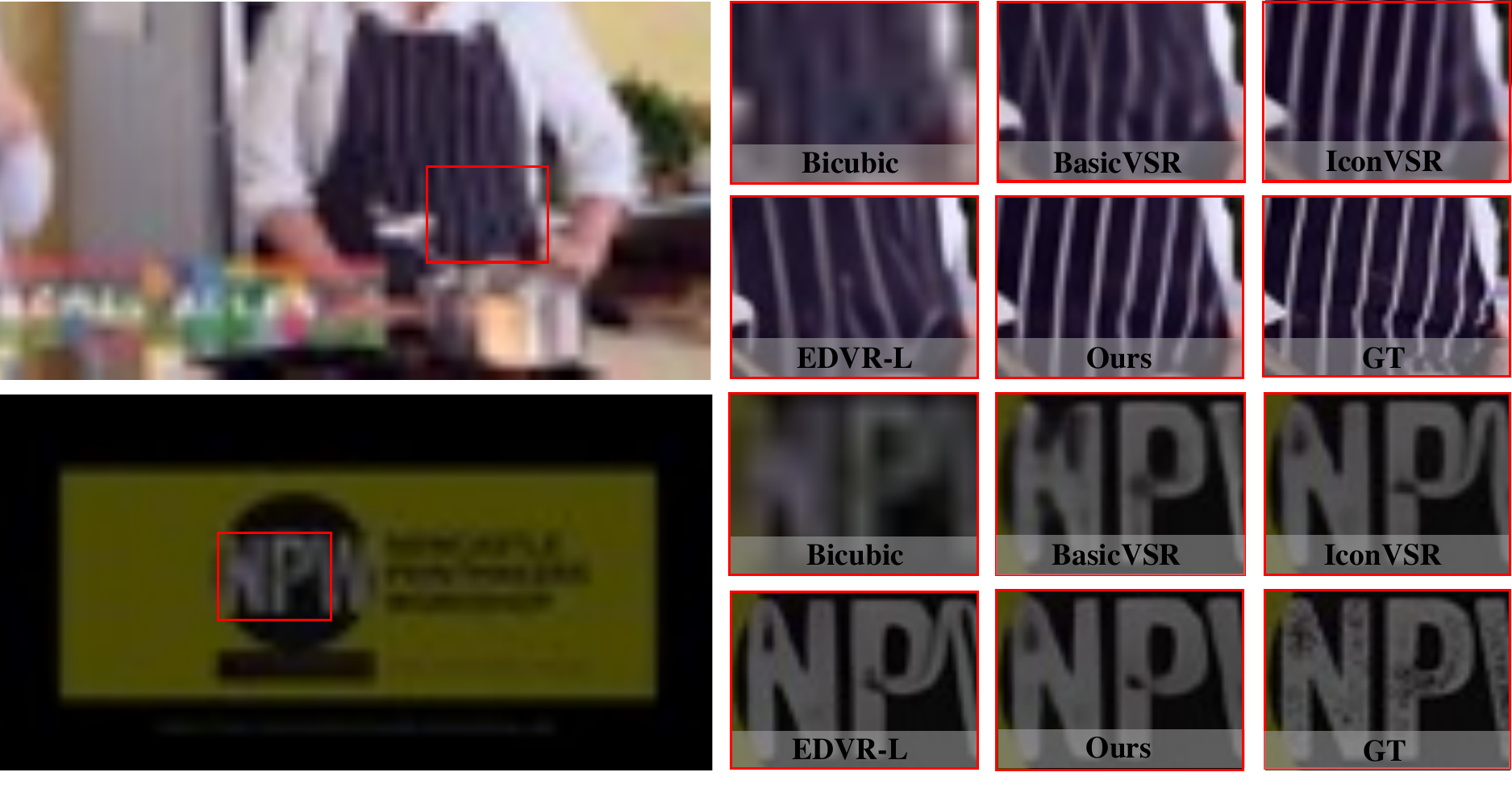}
	\caption{Qualitative comparison on Vimeo-90K-T for 4× VSR. Zoom in for the best view.
	}
	\label{fig: vimeo}
}
\end{figure}

\begin{figure}[t]
\centering
{
	\includegraphics[width=1\linewidth]{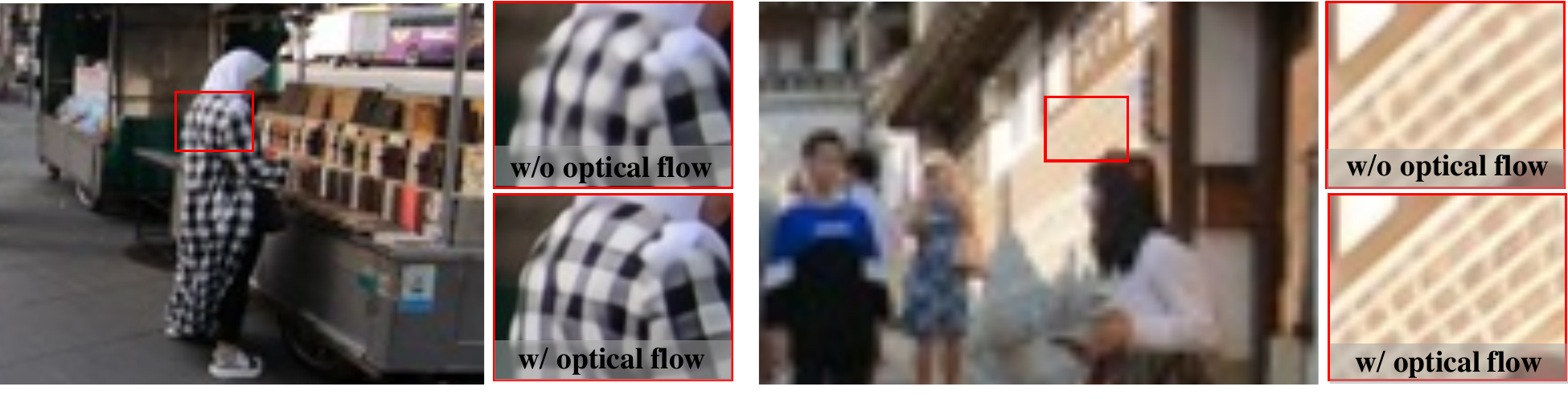}
	\caption{Ablation study on REDS for $4\times$ VSR. Here, w/o and w/ optical flow mean the VSR-Transformer without and with the optical flow, respectively. Zoom in for the best view.
	}
	\label{fig: ablation}
}
\end{figure}

\textbf{Qualitative results.}
From Figure \ref{fig: reds1}, our VSR-Transformer is able to recover finer details and sharper edges, including square patterns, the horizontal and vertical strip patterns.
In contrast, when training on 5 frames of REDS, BasicVSR and IconVSR are worse than EDVR and fail to generate sharp images.
Therefore, these results also verify the superiority of our method on the VSR task.
More qualitative results on REDS are shown in the supplementary materials.

\subsection{Resuts on Vimeo-90K}
In this experiment, we evaluate the performance of the VSR-Transformer on Vimeo-90K-T.
We train our all models on Vimeo-90K and then evaluate them on Vimeo-90K-T and Vid4.

\textbf{Quantitative results.}
From Table \ref{tab:sr_vimeo90k}, the VSR-Transformer achieves the highest PSNR and SSIM, and thus outperforms other VSR methods although the model size is larger than other methods. 
Here, the model size in Table \ref{tab:sr_vid4} is different from Table \ref{tab:reds} because the model size equals to the number of frames. 
In contrast, BasicVSR is much worse than EDVR and our method since the number of frames is small. 
When testing on Vid4, the generalization ability of our model is better than EDVR and is worse than BasicVSR and IconVSR.
The possible reason is that BasicVSR and IconVSR are tested on all frames of the Vid4, while the VSR-Transformer and EDVR are tested on 7 frames.
Moreover, there may exist a distribution bias between Vimeo-90K-T and Vid4.

\textbf{Qualitative results.}
As shown in Figure \ref{fig: vimeo}, the VSR-Transformer is able to generate sharp and realistic HR frames.
In contrast, BasicVSR and IconVSR often produce blurry HR images because of a small number of frames.
In addition, the texture generated by EDVR is blurry and messy. 
More qualitative results on the Vimeo-90K-t and Vid4 datasets are shown in the supplementary materials.

\subsection{Ablation Study}
We investigate the effectiveness of optical flows in our VSR-Transformer on REDS.
By removing SPyNet in the bidirectional optical flow-based feed-forward layer, we directly use a stack of Residual ReLU networks in the experiment. 
For the quantitative comparison, the VSR-Transformer without optical flows has the PSNR of 30.37, which has worse performance than that with optical flows.  
From Figure \ref{fig: ablation}, the VSR-Transformer with optical flow is able to generate HR frames with finer details and sharper edges.
It means that the optical flow is important in the bidirectional optical flow-based feed-forward layer and it helps to perform feature propagation and alignments.
More ablation studies can be found in the supplementary materials.

\section{Conclusion}
In this paper, we have proposed a novel Transformer framework for video super-resolution, namely VSR-Transformer.
Instead of directly applying existing vision Transformer for VSR, we present a spatial-temporal convolutional self-attention layer to leverage locality information.
Moreover, we provide a theoretical analysis to verify that the spatial-temporal convolutional self-attention layer has an advantage over the fully connected self-attention layer. 
Then, we develop a bidirectional optical flow-based feed-forward layer to exploit the correlations among different frames.
With the help of this layer, we are able to perform both feature propagation and alignment. 
Extensive experiments on several benchmark datasets demonstrate the effectiveness of our proposed method.


\bibliographystyle{abbrv} 



\newpage
\appendix

\def\blfootnote{\let\thefootnote\relax\footnotetext}

\begin{table}
	\setlength{\tabcolsep}{0.2cm}
	\begin{tabular}{p{0.97\columnwidth}}
		\nipstophline
		\rule{0pt}{1.0cm}
		\centering
		\!\!\!\!\Large{\textbf{Supplementary Materials: Video Super-Resolution Transformer\!\!\!\!\!}}
		\vspace{2pt}
	\end{tabular}
\end{table}
\begin{table}
	\begin{tabular}{c}
		\nipsbottomhline
		~~~~~~~~~~~~~~~~~~~~~~~~~~~~~~~~~~~~~~~~~~~~~~~~~~~~~~~~~~~~~~~~~~~~~~~~~~~~~~~~~~~~~~~~~~~~~~~~~~~~~~~~~~~~~~~~~~~~~~~~~~~~~~~~~~~~~~~~~~~~~~~~~~~~~~~~~~
	\end{tabular}
\end{table}


\setcounter{section}{0}
\renewcommand\thesection{\Alph{section}}
~\\
\paragraph{Organization.}
In the supplementary materials, we provide detailed proofs for all theorems and lemmas of our paper, and more experiment settings and results.
We organize our supplementary materials as follows.
For the theory part, we provide preliminaries of the proposed method in Section \ref{suppsec:pre}. In Sections \ref{suppsec:thm1} and \ref{suppsec:thm2}, we provide detailed proofs of Theorem \ref{thm:fcsa} and Theorem \ref{thm:stcsa}, respectively.
For Experiment part, we provide more experiment details and network architectures in Section \ref{suppsec:exp_detail}. 
In Section \ref{suppsec:ablation}, we provide more ablation studies for our proposed method.
In Section \ref{suppsec:qualitative}, we provide more qualitative results on several benchmark datasets.
In Section \ref{suppsec:vid4}, we provide more quantitative results on the Vid4 dataset.

\section{Preliminaries} \label{suppsec:pre} 
\paragraph{Notation.}
Throughout the paper, we use the following notations. 
We use a calligraphic letter $\mX$ to denote a sequence data, a bold upper case letters $\bX$ to denote a matrix, a bold lower case letter $\bx$ to denote a vector, a lower case letter $x$ to denote a element of a matrix. 
Let $\sigma_1(\cdot)$ be the softmax operator applied to each column of the matrix, \ie the matrix has non-negative elements with each column summing to 1.
Let $\sigma_2(\cdot)$ be a ReLU activation function, 
and let $\phi(\cdot)$ be a layer normalization function.
Let $[T]$ be a set $\{1, \ldots, T\}$.
Let $\1\{\cdot\}$ be an indicator function, where $ \1{\{A\}}=1 $ if $A$ is true and $ \1{\{A\}}=0 $ if $A$ is false.
Let $\mmE_{\mD}[\cdot]$ be an expectation with respect to the distribution $\mD$.

To develop our method, we first give some definitions of the function distance and $k$-pattern function.
\begin{deftn}\textbf{\emph{(Function distance)}}
    Given a functions $f: \mmR^{d{\times}n} \to \mmR^{d{\times}n}$ and a target function $f^*: \mmR^{d{\times}n} \to \mmR^{d{\times}n}$, we define a distance between these two function as:
    \begin{align}
        \mL_{f^*, \mD}(f):= \mmE_{\bX \sim \mD} \left[ \ell(f(\bX), f^*(\bX)) \right].
    \end{align}
\end{deftn}
For a ground-truth $\bY = f^*(\bX)$, we denote the loss as $\mL_{\mD}(f)$.
In the proofs of the theorem, we use a hing-loss $\ell(\hat{y}, y)=\max \{1-\hat{y} y, 0\}$ as an example.
To capture the locality of data, we define the $k$-pattern function as follows.

\begin{deftn}\textbf{\emph{($k$-pattern function \cite{ref_malach2021computational})}}
    A function $f{:} \mX {\to} \mY$ is a $k$-pattern if for some $g{:} \{{\pm1}\}^{k} {\to} \mY$ and index $j^*${:} $f(\bx) = g(x_{j^* \ldots j^*{+}k})$.
    We call a function $h_{\bu, \bW}(\bx) = \sum_{j} \langle \bu^{(j)}, \bv^{(j)}_{\bW} \rangle$ can learn a $k$-pattern function from a feature $\bv^{(j)}_{\bW}$ of data $\bx$ with a layer $\bu^{(j)}\in\mmR^{q}$ if for $\epsilon>0$, we have 
    \begin{align}
        \mL_{f, \mD}(h_{\bu, \bW}) \leq \epsilon.
    \end{align}
\end{deftn}
Note that the feature $\bv^{(j)}_{\bW}$ can be learned by a convolutional network or a fully-connected network followed by a ReLU activation function $\sigma_2(\cdot)$.
With the same linear layer $\bu$, if the convolutional network or the fully-connected network can learn the local pattern of data if $\mL_{f, \mD}(h_{\bu, \bW}) \leq \epsilon$.

We use the following theorem about the convergence of online gradient-descent.
This theorem verifies that the gradient-descent converges to a good solution.
\begin{thm} \textbf{\emph{(Online Gradient Descent \cite{ref_daniely2020learning})}} \label{thm:online_GD}
    Fix some $\eta$, and let $f_1, \dots, f_S$ be some sequence of convex functions.
    Fix some $\theta_1$, and update $\theta_{s+1} {=} \theta_{s} {-} \eta \nabla f_s (\theta_s)$.
    For every $\theta^*$ the following holds:
    \begin{align}
        \frac{1}{S} \sum_{s=1}^S f_s(\theta_s) {\leq} \frac{1}{S} \sum_{s=1}^S f_s(\theta^*) {+} \frac{1}{2 \eta S} \| \theta^* \|^2 {+} \| \theta_1 \| \frac{1}{S} \sum_{s=1}^S \| \nabla f_s (\theta_s) \| {+} \eta \frac{1}{S} \sum_{s=1}^S \| \nabla f_s (\theta_s) \|^2.
    \end{align}
\end{thm}

\newpage
\section{Proofs of Theorem \ref{thm:fcsa}} \label{suppsec:thm1}
\textbf{Theorem \ref{thm:fcsa}} 
\emph{
    Let $n$ be the size of image and $q$ be the size of $u$. We assume $m=1$ and $|u_i|\leq 1$. and the weights are initialized as some permutation invariant distribution over $\mmR^n$, and for all $\bx$ we have $h_{\bu, \bW}^{\emph{FCSA}}(\bx) \in [-1, 1]$ which satisfies Definition \ref{def:k_pattern}. Then, the following holds:
    \begin{align}
        \mmE_{\bW \sim \mW} \left\| \frac{\partial}{\partial \bW} \mL_{f, \mD}\left(h_{\bu, \bW}^{\emph{FCSA}}\right) \right\|_2^2 
        &\leq q n \min\left\{\begin{pmatrix} n{-}1 \\ k \end{pmatrix}^{{-}1}, \begin{pmatrix} n{-}1 \\ k{-}1 \end{pmatrix}^{{-}1}  \right\}.
    \end{align}
}
\begin{proof}
    Follows by the proofs of \cite{ref_malach2021computational}, we complete the following proofs.
    Denote $\chi_{\mI'} = \Pi_{i \in \mI'} x_i$, so $f(\bx) = \chi_{\mI}$ with $\mI = [k]$.
    By calculating the gradient to $\bw_j^{(i)}$:
    \begin{align*}
        \frac{\partial}{\partial \bw_j^{(i)}} \mL_{f, \mD} \left(h_{\bu, \bW}^{\emph{FCSA}}\right) 
        =& \mmE_{\bx \sim \mD} \left[ \frac{\partial}{\partial \bw_j^{(i)}} \ell\left( h_{\bu, \bW}^{\emph{FCSA}}, f(\bx) \right) \right] \\
        =& -\mmE_{\bx \sim \mD} \left[ x_j u_i \sigma_2' \left( \left[\bW_o (\bW_v \bx) \sigma_1\left((\bW_k \bx)^{\top}(\bW_q \bx) \right]_i \right) \right)\chi_{\mI}(\bx) \right] \\
        =& -\mmE_{\bx \sim \mD} \left[ x_j u_i \sigma_2' \left( \left\langle \bw^{(i)}, \bx \right\rangle \right) \chi_{\mI}(\bx) \right],
    \end{align*}
    where the last equation follows by the assumption $m{=}1$ and the loss function $\ell(\cdot, \cdot)$, and there exists $\bw^{(i)}$ such that the second line is satisfied.
    Fix some permutation $\pi: [n]{\to}[n]$. 
    Let $\pi(\bx) {=} (x_{\pi(1)}, \ldots, x_{\pi(n)})$, for $\mI {\subseteq} [n]$, we let $\pi(\mI) = \cup_{j \in \mI} \{ \pi(j) \}$.
    Notice that for all $\bx, \bz: \chi_{\mI}(\pi(\bx)) {=} \chi_{\pi(\mI)}$ and $\langle \pi(\bx), \bz \rangle {=} \langle \bx, \pi^{-1} (\bz) \rangle$.
    Denote $\pi(h_{\bu, \bW}^{\emph{FCSA}})(\bx) = \sum_{i=1}^k u_i \sigma_2( \langle \pi(\bw^{(i)}), \bx\rangle)$, and denote $\pi(\mD)$ the distribution of $\pi(\bx)$ where $\bx {\sim} \mD$.
    Notice that since $\mD$ is the uniform distribution, we have $\pi(\mD) = \mD$.
    Therefore, for every permutation $\pi$ with $\pi(j)=j$ we have:
    \begin{align*}
        - \frac{\partial}{\partial \bw_j^{(i)}} \mL_{\chi_{\pi(\mI)}, \mD} \left(h_{\bu, \bW}^{\emph{FCSA}}\right) 
        =& \mmE_{\bx \sim \mD} \left[ x_j u_i \sigma_2' \left( \left\langle \bw^{(i)}, \bx \right\rangle \right) \chi_{\pi(\mI)}(\bx) \right] \\
        =& \mmE_{\bx \sim \pi(\mD)} \left[ x_j u_i \sigma_2' \left( \left\langle \bw^{(i)}, \pi^{-1}(\bx) \right\rangle \right) \chi_{\mI}(\bx) \right] \\
        =& \mmE_{\bx \sim \mD} \left[ x_j u_i \sigma_2' \left( \left\langle \pi\left(\bw^{(i)}\right), \bx \right\rangle \right) \chi_{\mI}(\bx) \right] 
        = - \frac{\partial}{\partial \bw_j^{(i)}} \mL_{\chi_{\mI}, \mD} \left(\pi( h_{\bu, \bW}^{\emph{FCSA}})\right).
    \end{align*}
    Fix some $\mI \subseteq [n]$ with $|\mI| = k$ and $j \in [n]$.
    Now, let $\mP_j$ be a set of permutations satisfying: 
    (i) For all $\pi_1, \pi_2 \in \mP_j$ with $\pi_1 \neq \pi_2$ we have $\pi_1(\mI) \neq \pi_2 (\mI)$;
    (ii) For all $\pi \in \mP_j$ we have $\pi(j)=j$.
    Note that if $j \notin \mI$ then the maximal size of such $\mP_j$ is 
    $ \binom{n-1}{k}$, and if $j \in \mI$ then the maximal size is $ \binom{n-1}{k-1}$.
    Denote $g_j(\bx) = x_j u_i \sigma'_2(\langle \bw^{(i)}, \bx \rangle)$.
    We denote the inner-product $\langle \psi, \phi \rangle_{\mD} = \mmE_{\bx \sim \mD} [\psi(\bx) \phi(\bx)]$ and the induced norm $\| \psi \|_{\mD} = \sqrt{\langle \psi, \psi \rangle_{\mD}}$.
    Since $\{\chi_{\mI'}\}_{\mI' \subseteq [n]}$ is an orthonormal basis with respect to $\langle \cdot, \cdot \rangle_{\mD}$ from Parseval's equality, we have
    \begin{align*}
        \sum_{\pi \in \mP_j} \left( \frac{\partial}{\partial \bw_j^{(i)}} \mL_{\chi_{\mI}, \mD}\left( \pi \left( h_{\bu, \bW}^{\emph{FCSA}} \right) \right) \right)^2 
        =& \sum_{\pi \in \mP} \left( \frac{\partial}{\partial \bw_j^{(i)}} \mL_{\chi_{\pi(\mI)}, \mD}\left( h_{\bu, \bW}^{\emph{FCSA}} \right) \right)^2 \\
        =& \sum_{\pi \in \mP} \langle g_j, \chi_{\pi(\mI)} \rangle_{\mD}^2 \leq \sum_{\mI' \subseteq [n]} \langle g_j, \chi_{\mI'} \rangle_{\mD}^2 = \| g_j \|_{\mD}^2 \leq 1.
    \end{align*}
    So, from the above we get that, taking $\mP_j$ of maximal size:
    \begin{align*}
        \mmE_{\pi \sim \mP_j} \left( \frac{\partial}{\partial \bw_j^{(i)}} \mL_{\chi_{\mI}, \mD}\left( \pi \left( h_{\bu, \bW}^{\emph{FCSA}} \right) \right) \right)^2 \leq | \mP_j |^{-1} \leq \min\left\{\begin{pmatrix} n{-}1 \\ k \end{pmatrix}^{{-}1}, \begin{pmatrix} n{-}1 \\ k{-}1 \end{pmatrix}^{{-}1}  \right\}.
    \end{align*}
    Now, for some permutation invariant distribution of weights we have:
    \begin{align*}
        \mmE_{\bW} \left( \frac{\partial}{\partial \bw_j^{(i)}} \mL_{\chi_{\mI}, \mD}\left( h_{\bu, \bW}^{\emph{FCSA}} \right) \right)^2 
        = \mmE_{\bW} \mmE_{\pi \sim \mP_j} \left( \frac{\partial}{\partial \bw_j^{(i)}} \mL_{\chi_{\mI}, \mD}\left( \pi \left( h_{\bu, \bW}^{\emph{FCSA}} \right) \right) \right)^2 \leq | \mP_j |^{-1}.
    \end{align*}
    Summing over all neurons we get:
    \begin{align*}
        \mmE_{\bW} \left\| \frac{\partial}{\partial \bW} \mL_{\chi_{\mI}, \mD}\left( h_{\bu, \bW}^{\emph{FCSA}} \right) \right\|^2_2 \leq qn \min\left\{\begin{pmatrix} n{-}1 \\ k \end{pmatrix}^{{-}1}, \begin{pmatrix} n{-}1 \\ k{-}1 \end{pmatrix}^{{-}1}  \right\}.
    \end{align*}
\end{proof}

\newpage
\section{Proofs of Theorem \ref{thm:stcsa}}\label{suppsec:thm2}
\textbf{Theorem \ref{thm:stcsa}} 
\emph{
    Assume we initialize each element of weights uniformly drawn from $\{\pm 1/k\}$.
    Fix some $\delta >0$, some $k$-pattern $f$ and some distribution $\mD$. Then is $q>2^{k+3} \log (2^k/\delta)$, and let $h_{\bu^{(s)}, \mW^{(s)}}^{\emph{STCSA}}$ be a function satisfying Definition \ref{def:k_pattern},
    with probability at least $1-\delta$ over the initialization, when training a spatial-temporal convolutional self-attention (STCSA) layer using gradient descent with $\eta$, we have
    \begin{align}
        \frac{1}{S} \sum_{s=1}^{S} \mL_{f, \mD} \left(h_{\bu^{(s)}, \mW^{(s)}}^{\emph{STCSA}} \right) \leq \eta^2 S^2 nk^{5/2} 2^{k+1} + \frac{k^2 2^{2k+1}}{q \eta S} +  \eta nqk.
    \end{align}
}
\begin{proof}
    From Lemma \ref{lemma:separable}, with probability at least $1-\delta$ over the initialization, there exist $\bu^{*(1)}, \ldots, \bu^{*(n-k)}$ with $\| \bu^{*(1)}\| \leq 2^{k+1}k/\sqrt{q}$ and $\|\bu^{*(j)} \|=0$ for $j>1$ such that $h_{\bu^{*}, \mW^{(0)}}^{\emph{STCSA}}(\bx)=f(\bx)$, and so $\mL_{f, \mD} (h_{\bu^{(s)}, \mW^{(s)}}^{\emph{STCSA}} ) = 0$.
    Based on Theorem \ref{thm:online_GD}, since $\mL_{f, \mD} (h_{\bu, \mW}^{\emph{STCSA}})$ is convex with respect to $\bu$, we have:
    \begin{align*}
        \!\!\frac{1}{S} \!\!\sum_{s=1}^S \mL_{f, \mD}\!\left(h_{\bu^{(s)}, \mW^{(s)}}^{\text{STCSA}}\!\!\right)
        \leq& \frac{1}{S} \!\!\sum_{s=1}^S \mL_{f, \mD}\!\left(h_{\bu^*, \mW^{(s)}}^{\text{STCSA}}\!\right) \!\!{+} \frac{1}{2\eta S} \!\!\sum_{j=1}^{n-k} \left\| \bu^{*(j)} \right\|^2 \!\!\!\!{+} \eta \frac{1}{S} \!\!\sum_{s=1}^S \left\| \frac{\partial}{\partial \bu} \mL_{f, \mD} \left(h_{\bu^{(s)}, \mW^{(s)}}^{\text{STCSA}}\!\!\right) \right\|^2 \\
        \leq& \frac{1}{S} \sum_{s=1}^S \mL_{f, \mD}\left(h_{\bu^*, \mW^{(s)}}^{\text{STCSA}}\right) + \frac{2(2^k k)^2}{q \eta S} + \eta nqk \\
        \leq& \frac{1}{S} \sum_{s=1}^S \mL_{f, \mD}(h_{\bu^*, \mW^{(0)}}^{\text{STCSA}}) + \eta^2 S^2 n k^{3/2} \sqrt{q} \sum_{j=1}^{n-k} \left\| \bu^{*(j)} \right\| + \frac{2(2^k k)^2}{q \eta S} + \eta nqk\\
        \leq& \eta^2 S^2 nk^{5/2} 2^{k+1} + \frac{2(2^k k)^2}{q \eta S} + \eta nqk,
    \end{align*}
    where the first line follows by Theorem \ref{thm:online_GD}, and second line holds by the property of $\bu^{*(j)}$, the third line hold by Lemma \ref{lemma:Ls_L0}, and the fourth line follows by the inequality $\| \bu^{*(1)}\| \leq 2^{k+1}k/\sqrt{q}$.
\end{proof}

\begin{lemma}\label{lemma:Ls_L0}
    Given the learning rate $\eta$ and steps $S$, for every $\bu^*$, we have:
    \begin{align*}
        \left| \mL_{f, \mD} \left( h_{\bu^{*}, \mW^{(S)}}^{\emph{\text{STCSA}}} \right) - \mL_{f, \mD} \left( h_{\bu^{*}, \mW^{(0)}}^{\emph{\text{STCSA}}} \right) \right| \leq \eta^2 S^2 n k^{3/2} \sqrt{q} \sum_{j=1}^{n-k} \left\| \bu^{*(j)} \right\|.
    \end{align*}
\end{lemma}
\begin{proof}
    Based on the result of $\| \bW^{(S)} {-} \bW^{(0)} \|$ in Lemma \ref{lemma:WS_W0} and the assumption of $\sigma_2$, we have
    \begin{align*}
        &\left| \mL_{f, \mD} \left( h_{\bu^{*}, \mW^{(S)}}^{\text{STCSA}} \right) - \mL_{f, \mD} \left( h_{\bu^{*}, \mW^{(0)}}^{\text{STCSA}} \right) \right| \\
        =& \left| \mmE_{\bx \sim \mD} \left[ \ell(h_{\bu^{*}, \mW^{(S)}}^{\text{STCSA}}(\bx), f(\bx)) \right] - \mmE_{\bx \sim \mD} \left[ \ell(h_{\bu^{*}, \mW^{(0)}}^{\text{STCSA}}(\bx), f(\bx)) \right] \right| \\
        \leq& \mmE_{\bx \sim \mD} \left[ \left| h_{\bu^{*}, \mW^{(S)}}^{\text{STCSA}}(\bx) - h_{\bu^{*}, \mW^{(0)}}^{\text{STCSA}}(\bx) \right| \right] \\
        =& \mmE_{\bx \sim \mD} \left[ \left| \sum_{j=1}^{n-k} \left\langle \bu^{*(j)}, \sigma_2\left(\bW^{(S)}\bx_{j\ldots j+k}\right) - \sigma_2\left(\bW^{(0)}\bx_{j\ldots j+k}\right) \right\rangle \right| \right] \\
        \leq& \mmE_{\bx \sim \mD} \left[ \sum_{j=1}^{n-k} \left\| \bu^{*(j)} \right\| \left\| \bW^{(S)}-\bW^{(0)} \right\| \| \bx_{j\ldots j+k} \| \right] \\
        \leq& \eta^2 S^2 n k^{3/2} \sqrt{q} \sum_{j=1}^{n-k} \left\| \bu^{*(j)} \right\|,
    \end{align*}
    where the second and third lines follow by the definition of the loss function, the fourth line follows by the assumption of the activation, and the last line holds by Lemma \ref{lemma:WS_W0}.
\end{proof}

\newpage
\begin{lemma}\label{lemma:separable}
    Assume we initialize each element of weights uniformly drawn from $\{\pm 1/k\}$, and fix some $\delta>0$. 
    Then if $q>2^{k+3} \log(2^k / \delta)$ with probability at least $1-\delta$ over the choice of the weights, for every $k$-pattern $f$ there exist $\bu^{*(1)}, \ldots, \bu^{*(n-k)} \in \mmR^q$ with $\| \bu^{*(j^*)} \| \leq 2^{k+1}k/\sqrt{q}$ and $\| \bu^{*(j)} \| = 0$ for $j \neq j^*$ such that $h_{\bu^*, \mW}^{\emph{\text{STCSA}}} (\bx) = f(\bx)$.
\end{lemma}
\begin{proof}
    For some $\bz \in \{\pm 1\}^{k}$, then for every $\bw^{(i)} \sim \{\pm 1/k\}^{k}$, we have $ \mmP[\sign (\bw^{(i)})=\bz] = 2^{-k} $.
    Denote by $\Omega_{\bz} \subseteq [q]$ the subset of indexes satisfying $\sign(\bw^{(i)})=\bz$, for every $i\in \Omega_{\bz}$, and note that $\mmE_{\bw} |\Omega_{\bz}| \ge q 2^{-k}$.
    From Chernoff bound:
    \begin{align}
        \mmP[|\Omega_{\bz}| \leq q 2^{-k-1}] \leq e^{-q2^{-k}/8} \leq \delta 2^{-k}
    \end{align}
    by choosing $q>2^{k+3} \log (2^k / \delta)$. Thus, using the union bound with probability at least $1-\delta$, for every $\bz \in \{\pm 1\}^k$ we have $|\Omega_{\bz}|\ge q 2^{-k-1}$. Then, we have
    \begin{align*}
        \sigma_2 \left(\left\langle \bw^{(i)}, \bz \right\rangle \right)= \frac{1}{k} \1 \left\{\sign \left(\bw^{(i)}\right) = \bz\right\}.
    \end{align*}
    Fix some $k$-pattern $f$, where $f(\bx) = g(\bx_{j^* \ldots j^*+k})$.
    For every $i\in \Omega_{\bz}$ we choose $\bu_i^{*(j^*)} = \frac{k}{|\Omega_{\bz}|}g(\bz)$ and $\bu^{*(j)}=0$ for every $j\neq j^*$. Therefore, we have 
    \begin{align*}
        h_{\bu^*, \mW}^{\text{STCSA}} (\bx)
        &= \sum_{j=1}^{n-k} \left\langle {\bu^*}^{(j)}, \sigma_2\left(\left[\mW_o*((\mW_v*\mX) \sigma_1 \left( (\mW_k * \mX)^{\top} (\mW_q * \mX) \right) \right]_j  \right) \right\rangle\\
        &= \sum_{j=1}^{n-k} \left\langle {\bu^*}^{(j)}, \sigma_2\left(\left[(\mW_o*\mW_v)*\mX\right]_j \right)\right\rangle\\
        &=\sum_{j=1}^{n-k} \left\langle {\bu^*}^{(j)}, \sigma_2\left(\bW \bx_{j\ldots j{+}k} \right) \right\rangle \\
        &= \sum_{\bz \in \{\pm 1\}^k} \sum_{i \in \Omega_{\bz}} {\bu^*_i}^{(j^*)} \sigma_2\left(\left\langle \bw^{(i)}, \bx_{j^* \ldots j^*{+}k} \right\rangle \right) \\
        &= \sum_{\bz \in \{\pm 1\}^k} \1\{\bz = \bx_{j^* \ldots j^*{+}k}\} g(\bz) \\
        &= g(\bx_{j^* \ldots j^*{+}k}) \\
        &= f(\bx),
    \end{align*}
    where the first lines follows the definition of $h_{\bu^*, \mW}^{\text{STCSA}} (\bx)$ in Definition \ref{def:k_pattern}.
    The second line is based on the assumption $m=1$ such that $\sigma_1 ((\mW_k * \mX)^{\top} (\mW_q * \mX))=1$ and the property of convolution.
    The third line follows by the fact that there exists a weight $\bW$ such that $\bW \bx_{j^* \ldots j^*{+}k} = \left[(\mW_o*\mW_v)*\mX\right]_j$.
    The fourth line 
    Note that by definition of $\bu^{*(j^*)}$, we have . 
    
    \begin{align*}
        \left\| \bu^{*(j^*)} \right\|^2 = \sum_{\bz \in \{\pm 1\}^k} \sum_{i \in \Omega_{\bz}} \frac{k^2}{|\Omega_{\bz}|^2} \leq  \frac{4(2^k k)^2}{q}.
    \end{align*}
\end{proof}

\newpage
\begin{lemma}\label{lemma:WS_W0}
    Given the learning rate $\eta$ and steps $S$, the norm difference satisfies
    \begin{align*}
        \left\| \bW^{(S)} - \bW^{(0)} \right\| \leq \eta^2 S^2 n k \sqrt{q}.
    \end{align*}
\end{lemma}
\begin{proof}
    Based on the definition of $\mL_{f, \mD}$ and $h_{\bu^{(s)}, \mW^{(s)}}^{\text{STCSA}}$, we have
    \begin{align*}
        \left\| \frac{\partial}{\partial \bu^{(j)}} \mL_{f, \mD}\left(h_{\bu^{(s)}, \mW^{(s)}}^{\text{STCSA}}\right) \right\|
        =& \left\| \mmE_{\bx \sim \mD} \left[ \frac{\partial}{\partial \bu^{(j)}} \ell \left( h_{\bu^{(s)}, \mW^{(s)}}^{\text{STCSA}}, f(\bx) \right) \right]\right\| \\
        =& \left\|\mmE_{\bx \sim \mD} \left[ \sigma_2(\bW^{(s)} \bx_{j \ldots j+k}) \ell' \left( h_{\bu^{(s)}, \mW^{(s)}}^{\text{STCSA}}, f(\bx) \right) \right]\right\| \\
        \leq& \mmE_{\bx \sim \mD} \left[ \left\| \bW^{(s)} \bx_{j \ldots j+k} \right\| \right] \\
        \leq& \sqrt{qk}
    \end{align*}
    From the updates of gradient-descent we have:
    \begin{align*}
        \left\| \frac{\partial}{\partial \bW} \mL_{f, \mD}\left(h_{\bu^{(s)}, \mW^{(s)}}^{\text{STCSA}}\right) \right\| 
        =& \left\| \mmE_{\bx \sim \mD} \left[ \sum_{j=1}^{n-k} \bu^{(j, s)} \bx_{j \ldots j+k}^{\top} \sigma'_2(\bW^{(s)} \bx_{j \ldots j+k}) \ell' \left( h_{\bu^{(s)}, \mW^{(s)}}^{\text{STCSA}}, f(\bx) \right) \right] \right\|\\
        \leq& \left\| \mmE_{\bx \sim \mD} \left[ \sum_{j=1}^{n-k} \left\| \bu^{(j, s)} \right\| \left\| \bx_{j \ldots j+k} \right\| \right] \right\| \\
        \leq& \left\| \mmE_{\bx \sim \mD} \left[ \sum_{j=1}^{n-k} \sqrt{k} \left\| \eta \sum_{s=1}^S \frac{\partial}{\partial \bu^{(j)}} \mL_{f, \mD}\left(h_{\bu^{(s)}, \mW^{(s)}}^{\text{STCSA}}\right) \right\|  \right] \right\| \\
        \leq& \left\| \mmE_{\bx \sim \mD} \left[ \sum_{j=1}^{n-k} \sqrt{k} \eta \sum_{s=1}^S \left\| \frac{\partial}{\partial \bu^{(j)}} \mL_{f, \mD}\left(h_{\bu^{(s)}, \mW^{(s)}}^{\text{STCSA}}\right) \right\|  \right] \right\| \\
        \leq& (n-k)\eta S k\sqrt{q}
    \end{align*}
    By the updates of gradient-descent:
    \begin{align*}
        \left\| \bW^{(S)} {-} \bW^{(0)} \right\| 
        {=}& \left\| \eta \sum_{s=1}^S \frac{\partial}{\partial \bW} \mL_{f, \mD}\left(h_{\bu^{(s)}, \mW^{(s)}}^{\text{STCSA}}\right) \right\| \\
        {\leq}& \eta \sum_{s=1}^S \left\| \frac{\partial}{\partial \bW} \mL_{f, \mD}\left(h_{\bu^{(s)}, \mW^{(s)}}^{\text{STCSA}}\right) \right\| \\
        {\leq}& \eta^2 S^2 n k \sqrt{q}
    \end{align*}
\end{proof}

\newpage
\section{More Experiment Details and Network Architecture} \label{suppsec:exp_detail}
\textbf{More experiment details.}
The batch size is set to be 2 per GPU. 
We use Bicubic down-sampling to get LR images from HR images.
The channel size in each residual block is set to 64.
We set the number of Transformer blocks to be the number of frames.
We randomly crop a sequence of LR image patches with the size of $64{\times}64$.
We augment the training data with random horizontal flips and 90$^\circ$ rotations.
All frames are normalized to the fixed resolution 448$\times$256.
We use the pre-trained SPyNet as our flow estimation module.
Note that the SPyNet in our model is updated in the training.
We use Adam optimizer with $\beta_1{=}0.9, \beta_2{=}0.99$, and train our model with 60w iterations.
Then, we use Cosine Annealing to decay the learning rate from $2{\times}10^{-4}$ to $10^{-7}$.
On the REDS dataset, we set the periods as $[300000, 300000, 300000, 300000]$, the restart weights as $[1, 0.5, 0.5, 0.5]$.
On the Vimeo-90K dataset, we set the periods as $[200000, 200000, 200000, 200000, 200000, 200000]$, the restart weights as $[1, 0.5, 0.5, 0.5, 0.5, 0.5]$.
We use the Charbonnier loss in our method.

\textbf{Network architecture.}
We show the network architecture of the VSR-Transformer in Table \ref{table:vsr_net}.
In the spatial-temporal convolutional self-attention block, we use three independent convolutional neural networks to capture the spatial information of each frame.
Last, we use an output convolutional layer to obtain the final feature map.
In each bidirectional optical flow-based feed-forward block, we use 30 residual blocks for the backward and forward propagation networks, where $N$ is the number of frames. 
Last, we use a fusion layer to fuse the feature maps generated by the backward and forward propagation networks.
Here, the kernel size is $3 \times 3$, the stride is 1 and padding is 1.
The feature extractor has 5 residual blocks.
The resconstruction module has 30 residual blocks.
We use the following abbreviations: $T$: the number of frames, $C$: the number of channels, $H$: the height size of input image, $W$: the width size of input image, I: the the number of input channels, O: the number of output channels,  K: kernel size, S: stride size, P: padding size, G: groups, PixelShuffle: the pixel shuffle with the upscale factor of $2$,
LeakyReLU: the Leaky ReLU activation function with a negative slope of $0.01$. 

\begin{table}[h]
	\small
	\caption{Network architecture of the VSR-Transformer.}
	\label{table:vsr_net}
	\centering
	\resizebox{1\textwidth}{!}{
		\begin{tabular}{c|c|c}
			\hline
			\hline
			\multicolumn{3}{c}{\textbf{Spatial-temporal convolutional self-attention block}} \\
			\hline
			\hline 
			{Part} & {Input $ {\rightarrow} $ Output shape} & {Layer information} \\
			\hline 
			\multirow{1}[0]{*}{Query layer}  
			& $ (T, C, H, W) {\rightarrow} (T, C, H, W) $ & CONV-(I64, O64, K3x3, S1, P1, G64) \\	
			\multirow{1}[0]{*}{Key layer}  
			& $ (T, C, H, W) {\rightarrow} (T, C, H, W) $ & CONV-(I64, O64, K3x3, S1, P1, G64) \\	
			\multirow{1}[0]{*}{Value layer}  
			& $ (T, C, H, W) {\rightarrow} (T, C, H, W) $ & CONV-(I64, O64, K3x3, S1, P1, G64) \\	
			\hline 
			\multirow{1}[0]{*}{CNN layer}  
			& $ (T, C, H, W) {\rightarrow} (T, C, H, W) $ & CONV-(I64, O64, K3x3, S1, P1) \\	
			\hline
			\hline
			\multicolumn{3}{c}{\textbf{Bidirectional optical flow-based feed-forward block }} \\
			\hline
			\hline
			\multirow{1}[0]{*}{\tabincell{c}{Backward }}  
			& $ (T, C{+}3, H, W) {\rightarrow} (T, C, H, W) $ & Residual Block: CONV-(I67, O64, K3x3, S1, P1), LeakyReLU \\
			\hline
			\multirow{1}[0]{*}{\tabincell{c}{Forward}}
			& $ (T, C{+}3, H, W) {\rightarrow} (T, C, H, W) $ & Residual Block: CONV-(I67, O64, K3x3, S1, P1), LeakyReLU \\
			\hline
			\multirow{1}[0]{*}{\tabincell{c}{Fusion}}
			& $ (T, 2C, H, W) {\rightarrow} (T, C, H, W) $ & CONV-(I128, O64, K1x1, S1, P1), LeakyReLU \\
			\hline
			\hline
		\end{tabular}
	}
\end{table}

\begin{table}[h]
	\small
	\caption{Network architecture of the feature extractor and reconstruction network.}
	\label{table:generator_net}
	\centering
	\resizebox{1\textwidth}{!}{
		\begin{tabular}{c|c|c}
			\hline
			\hline
			\multicolumn{3}{c}{\textbf{Feature extractor}} \\
			\hline
			\hline 
			{Part} & {Input $ {\rightarrow} $ Output shape} & {Layer information} \\
			\hline 
			\multirow{1}[0]{*}{Extractor}  
			& $ (T, C, H, W) {\rightarrow} (T, C, H, W) $ & Residual Block: CONV-(I64, O64, K3x3, S1, P1), LeakyReLU \\	
			\hline
			\hline
			\multicolumn{3}{c}{\textbf{Reconstruction network}} \\
			\hline
			\hline
			\multirow{1}[0]{*}{\tabincell{c}{Reconstruction}}  
			& $ (T, C, H, W) {\rightarrow} (T, C, H, W) $ & Residual Block: CONV-(I64, O64, K3x3, S1, P1), LeakyReLU \\
			\hline
			\multirow{2}[0]{*}{\tabincell{c}{Upsampling}}
			& $ (T, C, H, W) {\rightarrow} (T, C, 2H, 2W) $ & CONV-(I64, O256, K3x3, S1, P1), PixelShuffle, LeakyReLU \\
			& $ (T, C, 2H, 2W) {\rightarrow} (T, C, 4H, 4W) $ & CONV-(I64, O256, K3x3, S1, P1), PixelShuffle, LeakyReLU \\
			\hline
			\multirow{1}[0]{*}{\tabincell{c}{CNN layer}}  
			& $ (T, C, 4H, 4W) {\rightarrow} (T, 3, 4H, 4W) $ & CONV-(I64, O3, K3x3, S1, P1), LeakyReLU \\
			\hline
			\hline
		\end{tabular}
	}
\end{table}

\newpage
\section{More Ablation Studies} \label{suppsec:ablation}

\subsection{Effectiveness of Spatial-Temporal Convolutional Self-Attention}
We investigate the effectiveness of spatial-temporal convolutional self-attention (STCSA) layer in our model on REDS.
Specifically, we remove this layer in our model, and evaluate the performance in Table \ref{tab:supp_ablation}.
For the quantitative comparison, the model without the STCSA layer has worse performance than the VSR-Transformer. 
From Figure \ref{fig: supp_ablation}, our model is able to generate HR frames with finer details and sharper edges.
It means that the STCSA layer is important in the VSR-Transformer and it helps to exploit the locality of data and fuse information among different frames.

\begin{table*}[!t]
	\caption{Ablation study (PSNR/SSIM) on \textbf{REDS4} for $4\times$ VSR. } 
	\label{tab:supp_ablation}
	\begin{center}
		\scalebox{0.92}{
			\begin{tabular}{l|c|c|c|c|c}
				\hline
				Method & Clip\_000 & Clip\_011 & Clip\_015 & Clip\_020 & Average (RGB)  \\ \hline
				w/o STCSA & 28.00/0.8247  & 32.00/0.8847  & 34.04/0.9189  & 30.14/0.8889 & 31.05/0.8793   \\
				w/o BOFF  & 27.67/0.8129  & 31.06/0.8683  & 33.39/0.9123  & 29.36/0.8729  & 30.37/0.8666   \\
				train w/ 3 frames & 27.59/0.8152  & 31.44/0.8747  & 33.64/0.9140  & 29.66/0.8809  & 30.58/0.8712   \\
				\hline
				\textbf{VSR-Transformer} &  {\bf{28.06}}/{\bf{0.8267}} & {\bf{32.28}}/{\bf{0.8883}} & {\bf{34.15}}/{\bf{0.9199}} & {\bf{30.26}}/{\bf{0.8912}} & {\bf{31.19}}/{\bf{0.8815}}  \\ \hline
			\end{tabular}
		}
		
	\end{center}
\end{table*}

\begin{figure}[t]
\centering
{
	\includegraphics[width=1\linewidth]{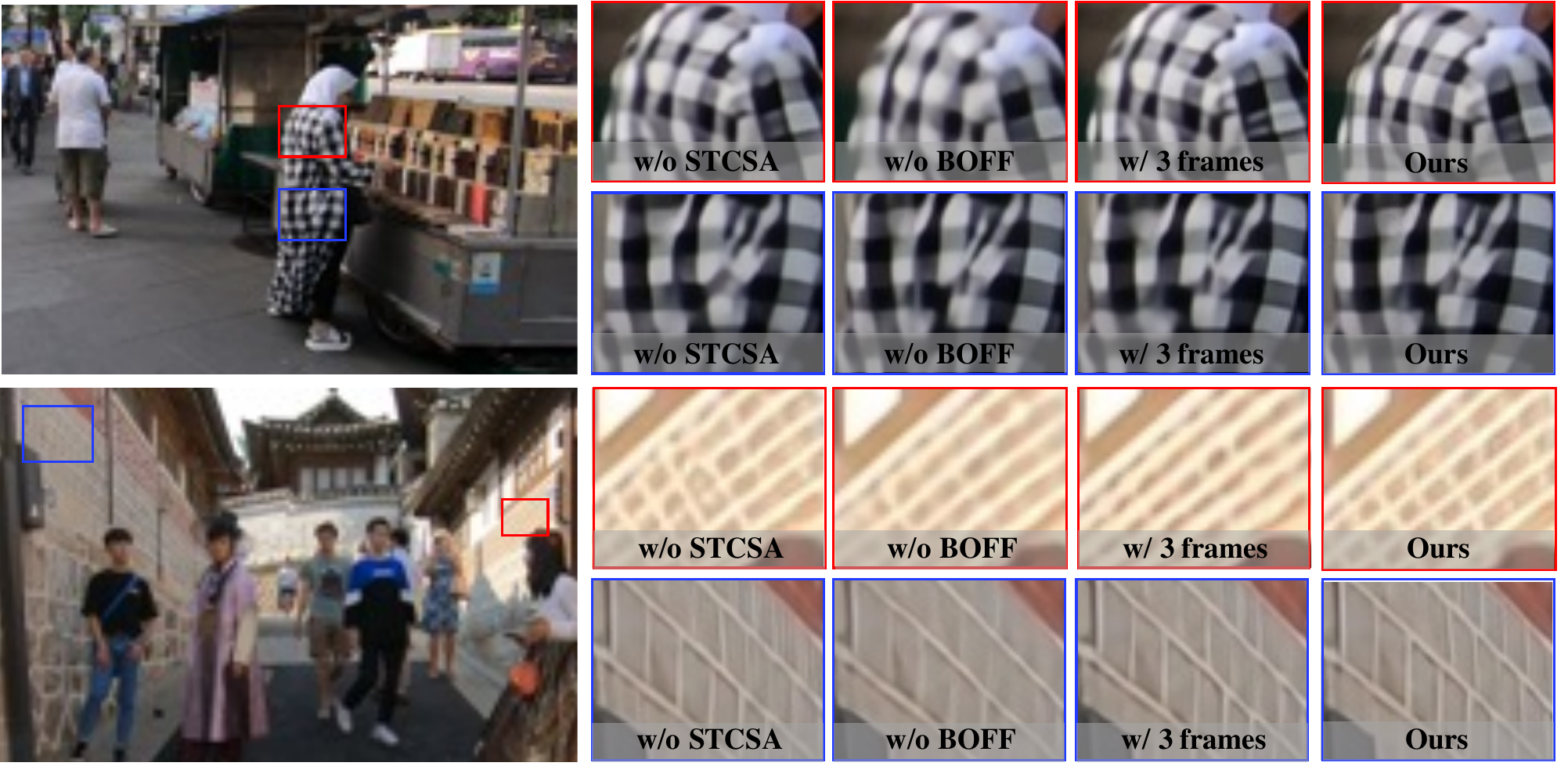}
	\caption{Ablation study on REDS for $4\times$ VSR. Here, w/o STCSA and w/o BOFF mean the VSR-Transformer without the  spatial-temporal convolutional self-attention (STCSA) layer and  bidirectional optical flow-based feed-forward (BOFF) layer, respectively. 
	}
	\label{fig: supp_ablation}
}
\end{figure}

\subsection{Effectiveness of Bidirectional Optical Flow-based Feed-Forward}
We investigate the effectiveness of bidirectional optical flow-based feed-forward (BOFF) and optical flows in our VSR-Transformer on REDS.
By removing this layer, we directly use a stack of Residual ReLU networks in the experiment. 
In Table \ref{tab:supp_ablation}, the model without the BOFF layer has worse performance than the VSR-Transformer.  
From Figure \ref{fig: supp_ablation}, the VSR-Transformer with optical flow is able to generate HR frames with finer details and sharper edges.
It means that the optical flow is important in the BOFF layer and it helps to perform feature propagation and alignments.

\subsection{Impact on Number of Frames}
We investigate the impact on the number of frames when training our VSR-Transformer on REDS.
Specifically, we train our model with 3 frames. 
In Table \ref{tab:supp_ablation}.
Training with small number of frames has degraded performance.
In contrast, our model is able to generate high-resolution frames, as shown in Figure \ref{fig: supp_ablation}.
Therefore, training with more frames helps to restore missing information from other neighboring frames.
In the future, we will train the VSR-Transformer with more frames.

\newpage
\section{More Qualitative Results} \label{suppsec:qualitative}
\subsection{Results on REDS4}

\begin{figure}[h]
\centering
{
	\includegraphics[width=1\linewidth]{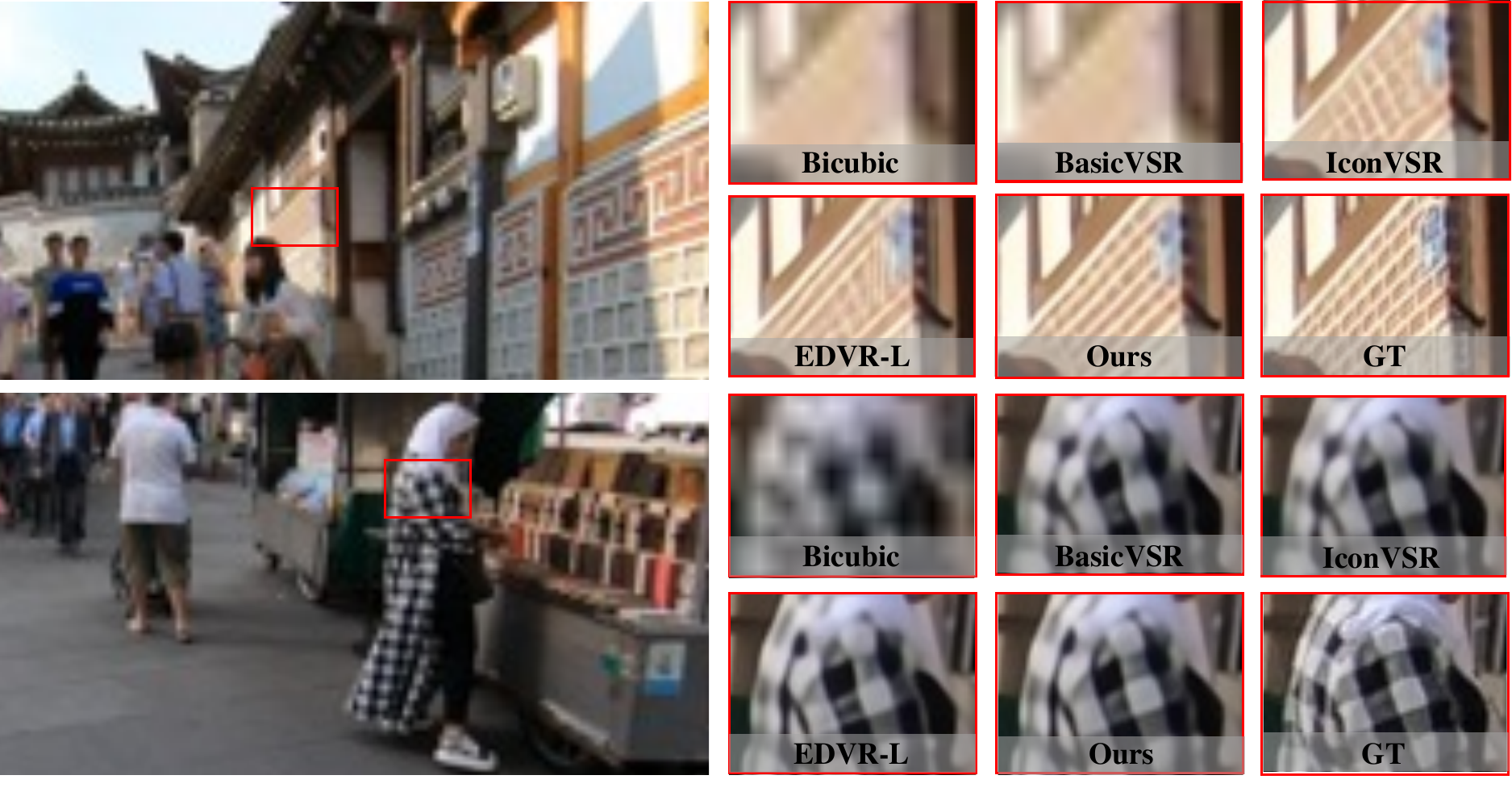}
	\caption{Qualitative comparison on the REDS4 dataset for $4\times$ VSR. Zoom in for the best view.
	}
	\label{fig: supp_reds}
}
\end{figure}

\subsection{Results on Vimeo-90K}
\begin{figure}[h]
\centering
{
	\includegraphics[width=1\linewidth]{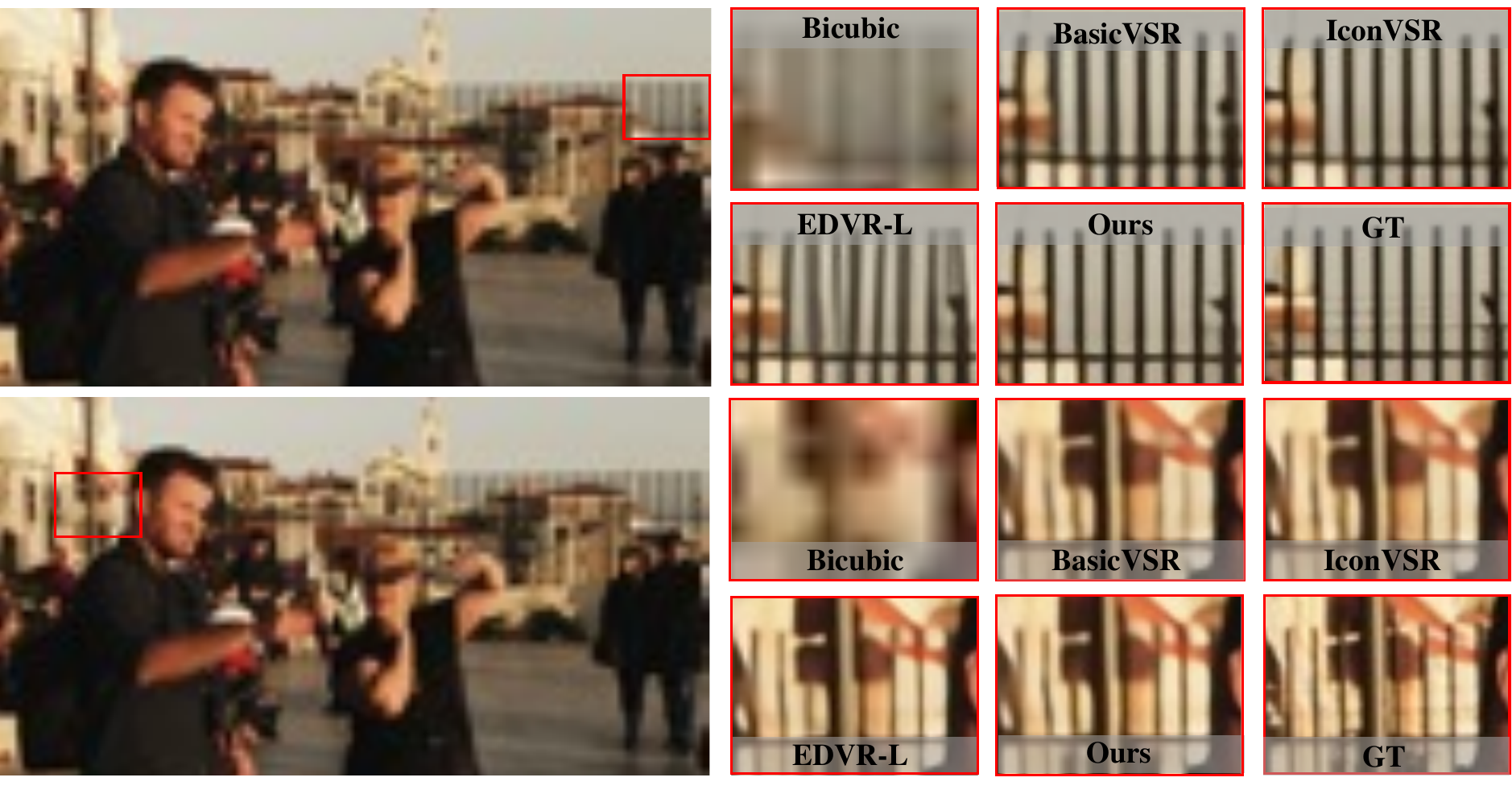}
	\caption{Qualitative comparison on Vimeo-90K-T for $4\times$ VSR. Zoom in for the best view.
	}
	\label{fig: supp_vimeo}
}
\end{figure}

\newpage
\subsection{Results on Vid4}
\begin{figure}[h]
\centering
{
	\includegraphics[width=1\linewidth]{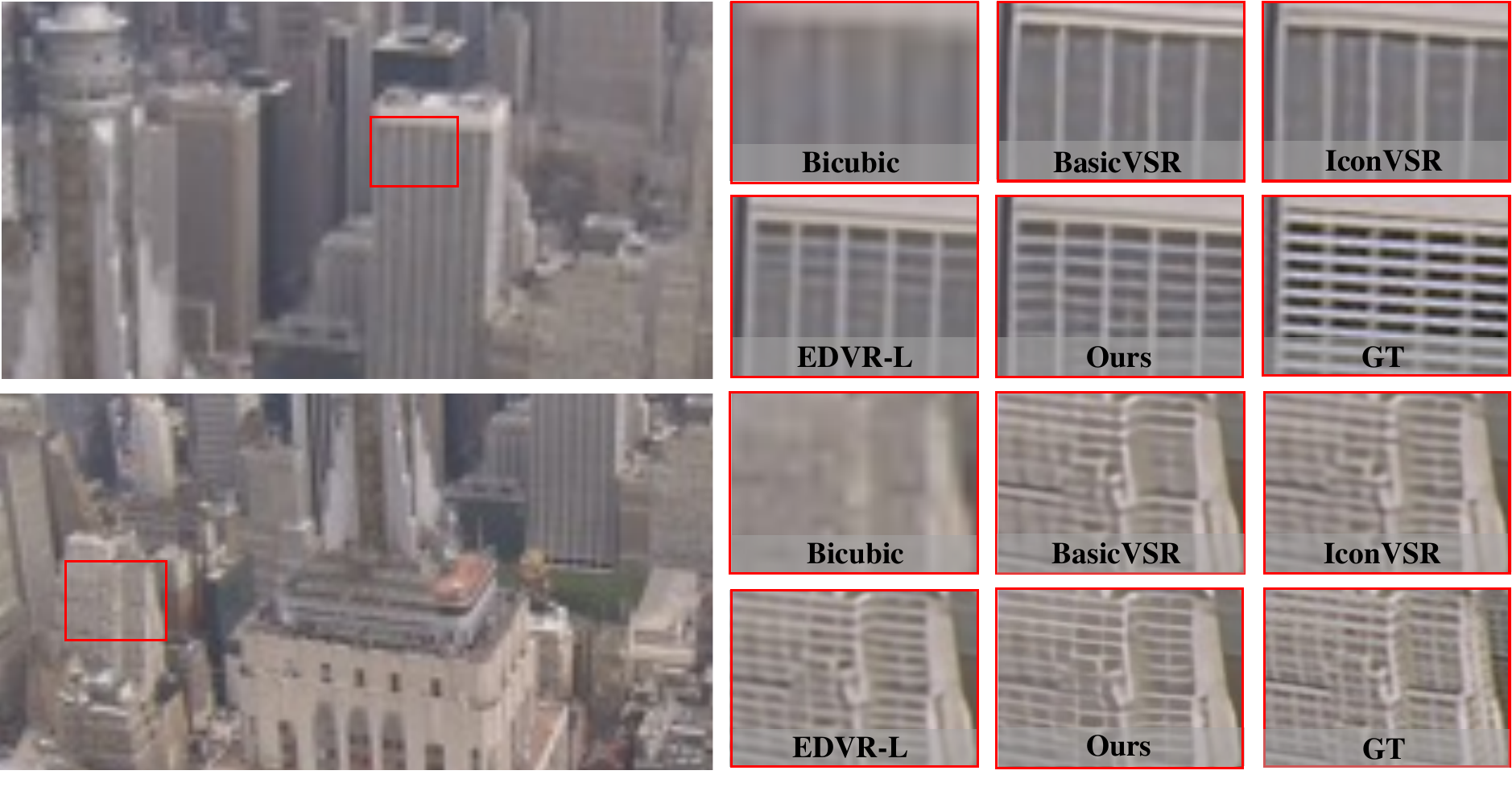}
	\caption{Qualitative comparison on Vid4 for $4\times$ VSR. Zoom in for the best view.
	}
	\label{fig: supp_vid}
}
\end{figure}

\section{More Results on Vid4}\label{suppsec:vid4}
From Table \ref{tab:supp_sr_vid4}, the VSR-Transformer achieves comparable PSNR and SSIM compared with EDVR-L. 
In contrast, BasicVSR and IconVSR are much worse than EDVR and our method since the number of frames is small. 
Note that Table \ref{tab:supp_sr_vid4} shows that BasicVSR and IconVSR are trained and tested on 7 frames, which is different from Table \ref{tab:sr_vid4} in the paper.
It implies that BasicVSR and IconVSR largely relies on the aggregation of long-term sequence information, and thus have poor performance especially when both training and testing on small number of frames.
These results verify the effectiveness and the generalization ability of our model.

\begin{table*}[h]
	\caption{Quantitative comparison (PSNR/SSIM) on \textbf{Vid4} for $4\times$ VSR. {Red} and {blue} indicate the best and the second best performance, respectively. Y denotes the evaluation on Y channels. `$\dagger$' means a method trained and tested on 7 frames for a fair comparison.}
	\label{tab:supp_sr_vid4}
	\begin{center}
		\tabcolsep=0.06cm
		\scalebox{0.9}{
		\begin{tabular}{l|c|c|c|c||c}
			\hline
			{Methods} & {Calendar (Y)} & {City (Y)} & {Foliage (Y)} & {Walk (Y)} & {Average (Y)} 
			\\
			\hline
			EDVR-L~\cite{ref_wang2019edvr}       &  \blue{\bf{24.05}}/\red{\bf{0.8147}} & \red{\bf{28.00}}/\red{\bf{0.8122}} & \red{\bf{26.34}}/\blue{\bf{0.7635}} & \blue{\bf{31.02}}/\blue{\bf{0.9152}} & \blue{\bf{27.35}}/\red{\bf{0.8264}} \\
			\hline
			BasicVSR$\dagger$~\cite{chan2020basicvsr}&   23.57/0.7905  &  27.60/0.7931    &   26.02/0.7485    &   30.42/0.9049   & 26.91/0.8093  \\
			IconVSR$\dagger$~\cite{chan2020basicvsr}&  23.76/0.7984  &  27.70/0.7997    &   26.13/0.7508    &   30.54/0.9072   & 27.04/0.8140  \\
			\hline
			\bf{VSR-Transformer (Ours)} &  \red{\bf{24.08}}/\blue{\bf{0.8125}}   &   \blue{\bf{27.94}}/\blue{\bf{0.8107}}    &   \blue{\bf{26.33}}/\red{\bf{0.7635}}    &   \red{\bf{31.10}}/\red{\bf{0.9163}}   &   \red{\bf{27.36}}/\blue{\bf{0.8258}}    \\
			\hline
		\end{tabular}
		}
	\end{center}
\end{table*}

\end{document}